\documentclass[11pt]{article}
\usepackage[letterpaper,margin=1in]{geometry}
\usepackage[T1]{fontenc}
\usepackage[utf8]{inputenc}
\usepackage[hyphens]{url}
\usepackage{graphicx}
\usepackage{natbib}
\usepackage{caption}
\usepackage{microtype}
\usepackage{amsmath,amsthm,amssymb,calrsfs,wasysym,verbatim,color,xcolor}
\usepackage{subcaption}
\usepackage{booktabs}
\usepackage{hhline}
\usepackage{makecell}
\usepackage{algorithm}
\usepackage{algorithmic}
\usepackage{pifont}
\usepackage[most]{tcolorbox}
\usepackage[colorlinks=true,
            hypertexnames=false,
            linkcolor=blue!50!black,
            citecolor=blue!50!black,
            urlcolor=blue!60!black]{hyperref}

\hypersetup{
  pdftitle={Variance-Aware Baselines and Adaptive Learning Rates for Reinforcement Learning with Verifiable Rewards},
  pdfauthor={Zixun Huang, Jiayi Sheng, Zeyu Zheng},
  pdfsubject={Reinforcement Learning with Verifiable Rewards},
  pdfkeywords={reinforcement learning, verifiable rewards, policy optimization, variance reduction, adaptive learning rates}
}

\newcommand{\dd}{\mathop{}\!\mathrm{d}}
\newcommand{\EE}{\mathbb{E}}

\newcommand{\bC}{\mathbf{C}}
\newcommand{\bH}{\mathbf{H}}

\newcommand{\cL}{\mathcal{L}}

\newcommand{\Var}{\operatorname{Var}}

\newcommand{\tr}{\operatorname{tr}}
\newtheorem{thm}{Theorem}

\newtheorem{lem}[thm]{Lemma}

\newtheorem{ass}{Assumption}

\newcommand{\cmark}{\textcolor{green!70!black}{\checkmark}}
\newcommand{\xmark}{\textcolor{red}{\ding{55}}}
\definecolor{oblrbaseline}{RGB}{31,95,155}
\definecolor{oblrlr}{RGB}{210,105,30}
\newcommand{\baselinecolor}[1]{\textcolor{oblrbaseline}{#1}}
\newcommand{\lrcolor}[1]{\textcolor{oblrlr}{#1}}
\newtcolorbox{takeaway}[1]{
  colback=blue!5!white,
  colframe=black,
  colbacktitle=blue!80!white,
  coltitle=white,
  fonttitle=\bfseries,
  title=Takeaway #1
}

\title{Variance-Aware Baselines and Adaptive Learning Rates for Reinforcement Learning with Verifiable Rewards}
\author{
Zixun Huang$^{1,*}$ \qquad
Jiayi Sheng$^{2,*}$ \qquad
Zeyu Zheng$^{3,\dagger}$\\[0.6em]
\normalsize $^1$University of Pennsylvania\\
\normalsize $^2$Carnegie Mellon University\\
\normalsize $^3$University of California, Berkeley\\
\small $^*$Equal contribution. $^\dagger$Corresponding author.
}
\date{}

\begin{document}
\maketitle

\begin{abstract}
Reinforcement learning with verifiable rewards (RLVR) has emerged as an effective paradigm for post-training large language models, yet the design of its baselines and learning-rate schedules remains largely heuristic. This limits our understanding of the statistical properties of policy-gradient estimators and their interaction with optimization dynamics. In this work, we develop a theoretical framework for variance-aware baseline design and adaptive learning-rate selection in RLVR. Under a KL-regularized policy-optimization setting, we establish the unbiasedness of the resulting gradient estimator, derive exact variance expressions including the KL cross-covariance, and obtain an optimization-loss upper bound that enables principled reasoning about learning dynamics. Building on these results, we prove convergence guarantees and derive an adaptive learning-rate schedule governed by the signal-to-noise ratio (SNR) of the policy gradient. We further show that the variance-optimal baseline is a gradient-weighted estimator of the KL-regularized reward, providing a principled alternative to commonly used reward-based baselines. 
These results lead to two complementary improvements: a variance-optimal baseline and an SNR-adaptive learning-rate rule. Experiments on Qwen3-4B-Base show that each component independently improves policy-optimization performance. The learning-rate rule can also be naturally integrated with existing policy optimization methods to yield further gains, while combining it with the variance-optimal baseline gives the full Optimal Baseline and Learning-Rate Policy Optimization (OBLR-PO) method and achieves the strongest overall performance.
\end{abstract}

\section{Introduction}
Reinforcement learning with verifiable rewards (RLVR) is widely used to improve language-model reasoning~\cite{shao2024deepseekmathpushinglimitsmathematical,deepseekai2025deepseekr1incentivizingreasoningcapability,sheng2025solving,kimiteam2026kimik3,deepseekai2026deepseekv4}. Its core policy-gradient update depends critically on two design choices: the baseline used to construct the gradient estimator and the learning rate used to scale each update. Existing algorithms instantiate these choices in different ways. For example, Proximal Policy Optimization (PPO)~\cite{schulman2018highdimensionalcontinuouscontrolusing} uses a learned value function, Group Relative Policy Optimization (GRPO)~\cite{deepseekai2025deepseekr1incentivizingreasoningcapability,shao2024deepseekmathpushinglimitsmathematical} uses group-normalized rewards, and REINFORCE with leave-one-out baselines (RLOO)~\cite{ahmadian2024basicsrevisitingreinforcestyle} constructs a baseline from the rewards of other sampled responses. Despite their empirical success, these designs are largely selected heuristically, leaving unclear how the baseline and learning rate should be chosen from a theoretical principle.

Existing baseline rules are typically functions of rewards or learned values~\cite{zhu2023finetuninglanguagemodelsadvantageinduced,ahmadian2024basicsrevisitingreinforcestyle,walder2025passkpolicyoptimizationsolving},  without considering how strongly each sample affects the policy gradient.
Because this influence can vary substantially across samples, a variance-minimizing baseline should also account for their gradient magnitudes~\cite{Greensmith2001VarianceRT}.
Meanwhile, commonly used fixed or predefined learning-rate schedules do not account for the changing reliability of the stochastic gradient during RLVR training. 
This motivates two related questions: what baseline minimizes the variance of the policy-gradient estimator, and how should the learning rate adapt to the estimator's signal and noise?

In this work, we develop a theoretical framework to answer these questions under a KL-regularized policy-optimization objective. We first establish the unbiasedness of the gradient estimator and derive its exact variance, including the cross-covariance induced by the KL term. We then obtain an upper bound on the optimization loss and use it to derive convergence guarantees. Optimizing this bound yields an adaptive learning-rate rule governed by the signal-to-noise ratio (SNR) of the policy gradient. Independently, minimizing the estimator variance yields a gradient-weighted baseline for the KL-regularized reward, rather than a conventional unweighted reward baseline.

These results give rise to two modular improvements: a variance-optimal baseline and an SNR-adaptive learning-rate rule. Each can be used independently and improves performance in our experiments. In particular, the learning-rate rule is agnostic to the choice of baseline and can be naturally added to existing policy optimization methods, including RLOO, ReMax, and GRPO, to provide further gains. The two components are also complementary: applying the adaptive learning rate together with the variance-optimal baseline produces the full \emph{Optimal Baseline and Learning-Rate Policy Optimization (OBLR-PO)} method and further improves performance. We validate the individual components and their combination by post-training Qwen3-4B-Base on mathematical reasoning tasks.

Our main contributions are as follows:
\begin{itemize}
    \item We present a theoretical framework for KL-regularized policy optimization that establishes unbiasedness, derives exact gradient-variance expressions, and provides an optimization-loss upper bound and convergence guarantees (Sections~\ref{sec:3}, \ref{sec:4.1}--\ref{sec:4.4}).
    \item We derive an SNR-adaptive learning-rate rule by optimizing the loss upper bound. This module can be applied to different policy-gradient estimators and baselines, and empirically improves RLOO, ReMax, and GRPO (Sections~\ref{sec:4.3}, \ref{sec:5.learning_rate}, and~\ref{sec:6.2}).
    \item We characterize the variance-optimal baseline as a gradient-weighted estimator of the KL-regularized reward and show that it improves gradient SNR and downstream accuracy relative to standard baselines (Sections~\ref{sec:4.4}, \ref{sec:5.baseline}, and~\ref{sec:6.3}).
    \item We combine the two complementary modules into OBLR-PO and show that their combination provides further gains over either component alone and existing policy optimization methods (Sections~\ref{sec:5} and~\ref{sec:6.4}).
\end{itemize}

\section{Related Work}
\label{sec:related_work}
\paragraph{Algorithmic Variants of Policy Optimization.}
Policy optimization is central to shaping the reasoning capabilities of large language models and has become a foundation of reinforcement learning from human feedback (RLHF)~\cite{christiano2023deepreinforcementlearninghuman,stiennon2022learningsummarizehumanfeedback,ouyang2022traininglanguagemodelsfollow,bai2022traininghelpfulharmlessassistant,deepseekai2025deepseekr1incentivizingreasoningcapability}. Within this framework, supervised fine-tuning (SFT)~\cite{ouyang2022traininglanguagemodelsfollow,wang-etal-2023-self-instruct} provides initial alignment through instruction imitation, PPO~\cite{schulman2018highdimensionalcontinuouscontrolusing} uses critic-based reinforcement learning on preference-model rewards, and Direct Preference Optimization (DPO)~\cite{rafailov2024directpreferenceoptimizationlanguage} directly matches preferences with a contrastive objective. To avoid the complexity of critic-based methods, GRPO~\cite{deepseekai2025deepseekr1incentivizingreasoningcapability,shao2024deepseekmathpushinglimitsmathematical} replaces the learned value function with a group-reward baseline, while ReMax~\cite{li2024remaxsimpleeffectiveefficient}, RLOO~\cite{ahmadian2024basicsrevisitingreinforcestyle}, and Reinforce++~\cite{hu2025reinforceefficientrlhfalgorithm} adopt maximum-reward, leave-one-out, or variance-reduced baselines, respectively. These designs are reward- or value-based; in contrast, we systematically analyze baseline choice and derive a gradient-weighted baseline for the KL-regularized reward by directly minimizing policy-gradient variance. The resulting practical estimator accounts for sample gradient magnitude and the cross-covariance induced by KL regularization.

\paragraph{Theoretical Foundations of Policy Optimization.}
Recent studies analyze policy-optimization dynamics through loss-function upper bounds and convergence guarantees~\cite{li2025stayingsweetspotresponsive,pang2025theorypracticegrpotrajectorycorrected,brantley2025acceleratingrlllmreasoning,yao2025optimizingchainofthoughtreasonersgradient}. Prior work studies the effect of a single-step update and the corresponding optimal update vector~\cite{li2025stayingsweetspotresponsive}, develops stochastic no-regret oracle frameworks with regret bounds and connections to online learning~\cite{brantley2025acceleratingrlllmreasoning}, and applies classical gradient-descent theory to establish convergence under smoothness and convexity assumptions~\cite{pang2025theorypracticegrpotrajectorycorrected,yao2025optimizingchainofthoughtreasonersgradient}. Related analyses provide convergence results for GRPO and related algorithms~\cite{pang2025theorypracticegrpotrajectorycorrected} and guaranteed loss decrease under smoothness~\cite{yao2025optimizingchainofthoughtreasonersgradient}. Complementing this line of work, we derive an exact gradient-variance characterization that retains the KL-induced cross-covariance and develop a unified optimization bound yielding both a variance-optimal baseline and an SNR-adaptive learning rate. Unlike predefined schedules, our learning-rate rule adapts to gradient reliability and can augment multiple policy-gradient estimators. Together, the two components control the variance and magnitude of policy updates and combine to form OBLR-PO.

\section{Problem Setup}
\label{sec:3}
\subsection{Objective Function}
In this section, we formally define our problem setup. Our target is to learn an optimal policy $\pi_\theta$ that maximizes the expected reward, which serves as a measure of accuracy or performance on the given task. Formally, we aim to solve:

\vspace{-1em}
\begin{align}
\label{eq:setup}
\max_\theta\quad J(\theta)
={}&\EE_{q\sim D,o\sim\pi_\theta(\cdot|q)}[F(q,o)]-\beta\EE_{q\sim D}\!\left[D_{\mathrm{KL}}\!\left(
\pi_\theta(\cdot|q)\,\|\,\pi_{\mathrm{ref}}(\cdot|q)\right)\right],
\end{align}


In online optimization algorithms, data is collected using an old policy \( \pi_{\theta_{\text{old}}} \), and importance sampling is employed to correct for the discrepancy between the old policy and the target policy \( \pi_\theta \) by scaling the advantage estimates with the importance sampling ratio $\frac{\pi_\theta(o | q)}{\pi_{\theta_{\text{old}}}(o | q)}$:
\vspace{-.5em}
\begin{align}
\label{eq:setup_surrogate}
\max_\theta\quad
&\EE_{q\sim D,o\sim\pi_{\theta_{\operatorname{old}}}(\cdot|q)}
\left[\frac{\pi_\theta(o|q)}{\pi_{\theta_{\operatorname{old}}}(o|q)}F(q,o)\right]-\beta\EE_{q\sim D}\!\left[D_{\mathrm{KL}}\!\left(
\pi_\theta(\cdot|q)\,\|\,\pi_{\mathrm{ref}}(\cdot|q)\right)\right].
\end{align}

Here $D$ denotes the distribution over queries $q$, $\pi_\theta$ is a behavior policy which generates outputs $o$ conditioned on $q$, $\pi_{\mathrm{ref}}$ is a fixed reference policy, and $\beta\geq0$ controls the KL regularization strength. $F(q,o)$ is an approximate reward which takes the query $q$ and output $o$ and returns a scalar value representing the estimated quality of the output. 

In our setup, the reward $F(q,o)$ is assumed to be available, either from accuracy supervision or a reward model.

\subsection{Related RL Algorithms}
\label{sec:3.2}
A variety of algorithms have been developed to improve the reasoning ability of large language models. 
In this section, we present several representative methods, each formulated through a distinct surrogate objective $J_{\operatorname{PO}}(\theta)$.

\paragraph{\underline{PPO}}
The Proximal Policy Optimization (PPO)~\cite{schulman2018highdimensionalcontinuouscontrolusing} objective is formulated as
\vspace{-.3em}
\begin{align}
    \label{eq:PPO}
    J_{\operatorname{PPO}}(\theta) 
    = &\EE_{q\sim D, \, o \sim \pi_{\theta_{\operatorname{old}}}(\cdot|q)} \left[ 
        \frac{\pi_\theta(o|q)}{\pi_{\theta_{\operatorname{old}}}(o|q)} 
        \, \hat{A}^{\mathrm{PPO}}(q,o) 
    \right]-\beta\EE_{q\sim D}\!\left[D_{\mathrm{KL}}\!\left(
\pi_\theta(\cdot|q)\,\|\,\pi_{\mathrm{ref}}(\cdot|q)\right)\right],
\end{align}
\vspace{-.3em}
where the advantage estimator $\hat{A}^{\operatorname{PPO}}$ is computed using 
Generalized Advantage Estimation (GAE) \cite{schulman2018highdimensionalcontinuouscontrolusing}, 
a widely adopted variance-reduction technique that stabilizes policy-gradient training.

\paragraph{\underline{GRPO}} Group Relative Policy Optimization (GRPO)~\cite{shao2024deepseekmathpushinglimitsmathematical} optimizes policies by comparing rewards among a group of sampled outputs, without relying on an explicit value or reward model. 
Given a query $q$, we draw $G$ outputs $\{o_i\}_{i=1}^G$ from the old policy $\pi_{\theta_{\operatorname{old}}}$ with associated rewards $r_i = F(q,o_i)$. 
The objective is
\vspace{-.3em}
\begin{align}
    J_{\operatorname{GRPO}}(\theta) 
    = &\EE_{q\sim D,\,\{o_i\}\sim \pi_{\theta_{\operatorname{old}}}} \left[ 
        \frac{1}{G}\sum_{i=1}^G 
        \frac{\pi_\theta(o_i|q)}{\pi_{\theta_{\operatorname{old}}}(o_i|q)} 
        \, \hat{A}^{\mathrm{GRPO}}(q,o_i) 
    \right]-\beta\EE_{q\sim D}\!\left[D_{\mathrm{KL}}\!\left(
\pi_\theta(\cdot|q)\,\|\,\pi_{\mathrm{ref}}(\cdot|q)\right)\right],
\end{align}
\vspace{-.3em}
where the group-relative advantage is normalized as
\vspace{-.3em}
\begin{equation}
    \label{eq:GRPO}
    \hat{A}^{\mathrm{GRPO}}(q,o_i) 
    = \frac{r_i - \frac{1}{G}\sum_{j=1}^G r_j}{\sqrt{\frac{1}{G}\sum_{j=1}^G\left(r_j - \frac{1}{G}\sum_{k=1}^G r_k\right)^2}}.
\end{equation}
\vspace{-.3em}

\paragraph{\underline{ReMax}}  
The ReMax~\cite{li2024remaxsimpleeffectiveefficient} method draws inspiration from the REINFORCE with Baseline approach, where we modify the gradient estimation by incorporating a subtractive baseline value. The objective is:
\vspace{-.3em}
\begin{align}
    J_{\operatorname{ReMax}}(\theta) 
    =  \EE_{q_i \sim D, \, o^i_{1:T} \sim \pi_{\theta}(\cdot | q_i)}  &\left[ 
        \frac{1}{N} \sum_{i=1}^{N} \sum_{t=1}^{T} 
        \frac{\pi_\theta(o_{t}^i|q_i, o_{1:t-1}^i)}{\pi_{\theta_{\operatorname{old}}}(o_{t}^i|q_i, o_{1:t-1}^i)} 
        \hat{A}^{\mathrm{ReMax}}(q_i, o^i_{1:T}) 
    \right]\nonumber\\&-\beta\EE_{q\sim D}\!\left[D_{\mathrm{KL}}\!\left(
\pi_\theta(\cdot|q)\,\|\,\pi_{\mathrm{ref}}(\cdot|q)\right)\right],
\end{align}
where the action \(o^i_t \sim \pi_{\theta}(\cdot | q_i, o^i_{1:t-1})\), and \(b_{\theta}(q_i)\) is the baseline value. The choice for the baseline is:
\vspace{-.3em}
\[
b_{\theta}(q_i) = r(q_i, \bar{o}^i_{1:T}), \quad \bar{o}^i_t \in \arg\max \pi_{\theta}(\cdot | q_i, \bar{o}^i_{1:t-1}),
\]

This baseline value is obtained by greedily sampling the response and calculating the associated reward value.

The advantage function is defined as:
\vspace{-.3em}
\[
\hat{A}^{\mathrm{ReMax}}(q_i, o^i_{1:T}) = r(q_i, o^i_{1:T}) - b_{\theta}(q_i).
\]
\vspace{-.5em}

\paragraph{\underline{RLOO}} 
REINFORCE Leave-One-Out (RLOO)~\cite{ahmadian2024basicsrevisitingreinforcestyle, Kool2019Buy4R} extends the REINFORCE estimator to the multi-sample setting by employing a leave-one-out baseline. 
Given a query $q$, we draw $G$ outputs $\{o_i\}_{i=1}^G$ from the old policy $\pi_{\theta_{\operatorname{old}}}$ with rewards $r_i = F(q,o_i)$. 
The objective is
\vspace{-.3em}

\begin{align}
    J_{\operatorname{RLOO}}(\theta) 
    = &\EE_{q\sim D,\,\{o_i\}\sim \pi_{\theta_{\operatorname{old}}}}\left[ 
        \frac{1}{G}\sum_{i=1}^G 
        \frac{\pi_\theta(o_i|q)}{\pi_{\theta_{\operatorname{old}}}(o_i|q)} 
        \, \hat{A}^{\mathrm{RLOO}}(q,o_i) 
    \right]-\beta\EE_{q\sim D}\!\left[D_{\mathrm{KL}}\!\left(
\pi_\theta(\cdot|q)\,\|\,\pi_{\mathrm{ref}}(\cdot|q)\right)\right],
\end{align}
\vspace{-.3em}
where the leave-one-out advantage is
\vspace{-.3em}
\begin{equation}
    \label{eq:RLOO}
    \hat{A}^{\mathrm{RLOO}}(q,o_i) 
    = r_i - \frac{1}{G-1}\sum_{j\neq i} r_j.
\end{equation}
\vspace{-.3em}

\paragraph{\underline{General Form}}
The KL-regularized surrogate objective $J_{\operatorname{PO}}(\theta)$ can be expressed as
\vspace{-.3em}
\begin{align}
J_{\operatorname{PO}}(\theta)
=&\EE_{q\sim D,\,\{o_i\}\sim \pi_{\theta_{\operatorname{old}}}}\left[\frac{1}{G_t}\sum_{i=1}^{G_t}
\frac{\pi_\theta(o_i|q)}{\pi_{\theta_{\operatorname{old}}}(o_i|q)}
\hat A^{\mathrm{PO}}(q,o_i)\right]-\beta\EE_{q\sim D}\!\left[D_{\mathrm{KL}}\!\left(
\pi_\theta(\cdot|q)\,\|\,\pi_{\mathrm{ref}}(\cdot|q)\right)\right].
\label{eq:general_beta}
\end{align}
Here $G_t$ denotes the group size at iteration $t$. Setting $\beta=0$ recovers the original formulation.

\subsection{Online Gradient Ascent}
We consider online gradient ascent as our training algorithm. As in the original analysis, we study the beginning of each policy update, where $\pi_{\theta_{\mathrm{old}}}=\pi_\theta$. For the forward KL,
\vspace{-.2em}
\begin{align}
&\nabla_\theta D_{\mathrm{KL}}\!\left(
\pi_\theta(\cdot|q)\,\|\,\pi_{\mathrm{ref}}(\cdot|q)\right)=\EE_{o\sim\pi_\theta(\cdot|q)}\!\left[
\nabla_\theta\log\pi_\theta(o|q)
\log\frac{\pi_\theta(o|q)}{\pi_{\mathrm{ref}}(o|q)}\right].
\end{align}
\vspace{-.3em}
Therefore, the gradient estimator is
\vspace{-.2em}
\begin{align}
\widehat{\nabla_\theta J(\theta)}
&={}\frac{1}{N_tG_t}\sum_{j=1}^{N_t}\sum_{i=1}^{G_t}
\nabla_\theta\log\pi_\theta(o_{i,j}|q_j)\left(\hat{A}^{\mathrm{PO}}(q_j, o_{i,j})
-\beta\log\frac{\pi_\theta(o_{i,j}|q_j)}
{\pi_{\mathrm{ref}}(o_{i,j}|q_j)}\right).
\label{eq:beta_gradient_estimator}
\end{align}
\vspace{-.6em}
The policy is then updated by
\begin{equation}
\theta_{t+1}=\theta_t+\eta_t\widehat{\nabla_\theta J(\theta_t)}.
\end{equation}
\vspace{-1em}

\subsection{Assumptions}
\label{sec:3.4}
For the simplicity of theoretical analysis, we require the assumption as below.
\begin{ass}
\label{ass:1}
Let \( o \sim \pi_\theta(\cdot|q) \) be an output sampled from the policy for a given query \( q \). We assume that the advantage is computed as
\vspace{-.3em}
\begin{equation}
  \hat{A}^{\mathrm{PO}}(q, o) = F(q, o) - b_\theta(q),  
\end{equation}
where \( b_\theta(q) \) denotes a reference value. For theoretical analysis, we treat \( b_\theta(q) \) as independent of the sampled output \( o \), though it may, in general, be related to other outputs of the query.
\end{ass}

Under Assumption~\ref{ass:1}, the baseline remains a zero-mean control variate:
\begin{equation}
\EE_{o\sim\pi_\theta(\cdot|q)}\!\left[
 b_\theta(q)\nabla_\theta\log\pi_\theta(o|q)\right]=0.
\end{equation}
Consequently, the estimator in Equation~\eqref{eq:beta_gradient_estimator} has expectation $\nabla_\theta J(\theta)$. Thus, we retain the original notation $J(\theta)$ and $\widehat{\nabla_\theta J(\theta)}$ throughout the analysis.

\begin{table}[ht]
\centering
\caption{Satisfaction of Assumption~\ref{ass:1} across algorithms.}
\begin{tabular}{lc}
\toprule
Algorithm & Assumption~\ref{ass:1} satisfied \\
\midrule
PPO~\cite{schulman2018highdimensionalcontinuouscontrolusing}            & \xmark \\
GRPO~\cite{shao2024deepseekmathpushinglimitsmathematical}           & \xmark \\
ReMax~\cite{li2024remaxsimpleeffectiveefficient}
& \cmark \\
RLOO~\cite{ahmadian2024basicsrevisitingreinforcestyle}            & \cmark \\
\midrule
OBLR-PO (Ours)  & \cmark \\
\bottomrule
\end{tabular}
\end{table}





\begin{ass}
\label{ass:2}
The KL-regularized objective $J(\theta)$ is $L$-smooth, i.e., for all $\theta,\theta'$,
\begin{equation}
\|\nabla J(\theta')-\nabla J(\theta)\|_2
\le L\|\theta'-\theta\|_2.
\end{equation}
\end{ass}

\begin{ass}
    \label{ass:4}
    We assume that there exists a uniform upper bound for the squared norm of the gradient of the log-likelihood, i.e.,
    \begin{equation}
        \int\sup_{\theta } \| \nabla_\theta [\log \pi_\theta (o|q)] \|_2^2 \, \mathrm{d}o \dd q \le M.
    \end{equation}
\end{ass}

\begin{ass}
\label{ass:3}
There exists a constant $B$ such that
\begin{equation}
\left|F(q,o)-\beta\log\frac{\pi_\theta(o|q)}
{\pi_{\mathrm{ref}}(o|q)}\right|\le B,
\qquad |b_\theta(q)|\le B.
\end{equation}
\end{ass}

\section{Main result}
\label{sec:theory}
\vspace{-0.5em}

\subsection{Bias and Variance Analysis}
\label{sec:4.1}
\vspace{-0.5em}
In this section, we prove that the KL-regularized gradient estimator is unbiased and has a tractable variance.

\begin{thm}[Unbiasedness]
The estimator in Equation~\eqref{eq:beta_gradient_estimator} is unbiased, i.e.,
\begin{equation}
\EE\!\left[\widehat{\nabla_\theta J(\theta)}\right]=\nabla_\theta J(\theta).
\end{equation}
\end{thm}
We retain the original decomposition
\begin{equation}
\widehat{\nabla_\theta J(\theta)}=\nabla_\theta J(\theta)+\xi(\theta),
\end{equation}
where $\EE[\xi(\theta)]=0$. The single-sample covariance is now
\vspace{-.3em}
\begin{align}
\bH(\theta):=\Var\Bigg[&\nabla_\theta\log\pi_\theta(o|q)
\Bigg(F(q,o)-b_\theta(q)-\beta\log\frac{\pi_\theta(o|q)}{\pi_{\mathrm{ref}}(o|q)}\Bigg)\Bigg],
\end{align}
and $\bC(\theta)$ denotes the covariance between two distinct samples of the same expression under a common query.

\begin{thm}[Variance Expression]
\label{thm:4.2}
The covariance matrix of $\xi(\theta)$ is
\begin{equation}
\Var[\xi(\theta)]
=\frac{1}{N_tG_t}\bH(\theta)
+\frac{G_t-1}{N_tG_t}\bC(\theta).
\end{equation}
\end{thm}
The proofs can be found in Appendix~\ref{app:A.1}.

\subsection{Deriving an Upper Bound for the Loss Function}
\label{sec:4.2}
In this section, we denote the loss function as \( \cL(\theta) = J(\theta^*) - J(\theta) \), where \( \theta^* \) is the optimal parameter. We then derive an upper bound for \( J(\theta^*) - J(\theta_T) \) in terms of the learning rate \( \{\eta_t\}_{t=0}^{T-1} \). It is clear that
\begin{equation}
    \min_{\theta} \cL(\theta) = \cL(\theta^*) = 0,
\end{equation}
indicating that the loss function attains its minimum value when \( \theta = \theta^* \), where $\cL(\theta)$ vanishes.

Our main result is stated in the following theorem, 
whose proof  can be found in Appendix \ref{app:A.2}.

\begin{thm}[Upper Bound]
    \label{thm:upperbound}
    Under Assumptions~\ref{ass:1}--\ref{ass:3} and Assumption~\ref{ass:4},
    \begin{align}
    \EE[\cL(\theta_T)]
    \le{}&\EE[\cL(\theta_0)]
    -\sum_{t=0}^{T-1}\eta_t\EE\|\nabla\cL(\theta_t)\|_2^2+\frac{L}{2}\sum_{t=0}^{T-1}\eta_t^2
    \left(\EE\|\nabla\cL(\theta_t)\|_2^2
    +\frac{1}{N_t}\tr(\bH(\theta_t))\right).
    \end{align}
\end{thm}

\subsection{Optimal Learning Rate Schedule}
\label{sec:4.3}
In this section, we aim to identify the optimal learning rate schedule \( \{\eta_t\}_{t=0}^{T-1} \) that minimizes the final expected loss \( \mathbb{E}[\cL(\theta_T)] \), formally stated as
\begin{equation}
    \min_{\{\eta_t\}_{t=0}^{T-1}} \ \EE\left[\cL(\theta_T)\right].
\end{equation}
However, since the exact evaluation of \( \mathbb{E}[\cL(\theta_T)] \) is generally intractable, we instead consider minimizing the upper bound provided in Theorem~\ref{thm:upperbound}. This leads to the following result.


\begin{thm}[Optimal Learning Rate Schedule]
    \label{thm:opt_lr}
    The optimal learning rate schedule is given by
    \begin{align}
        \eta_t &= \frac{1}{L}\cdot \frac{\EE\|\nabla_\theta \cL(\theta_t)\|_2^2}{\EE\|\nabla_\theta \cL(\theta_t)\|_2^2 + \frac{1}{N_t} \tr(\bH(\theta_t))} = \frac{1}{L} \cdot \frac{N_t \operatorname{SNR}(\theta_t)}{1+N_t \operatorname{SNR}(\theta_t)}.
    \end{align}
    Here, we introduce the concept of the signal-to-noise ratio to measure the information content of a stochastic gradient:
    \begin{equation}
        \operatorname{SNR}(\theta)
        =\frac{\EE\|\nabla_\theta \cL(\theta)\|_2^2}
        {\tr(\bH(\theta))}.
    \end{equation}
\end{thm}

This theorem shows that the optimal learning rate is governed by the
signal-to-noise ratio (SNR) of the gradient.
\begin{takeaway}{1}
Richer information in $\theta_t$ allows us to trust updates more and use a larger learning rate.
\end{takeaway}


By selecting the optimal learning-rate schedule, we derive the following upper bound:

\begin{thm}
\label{thm:5}
    Under the optimal learning rate schedule~\ref{thm:opt_lr}, we have
    \begin{align}
        &\quad\quad\mathbb{E} [\cL(\theta_T)] \le \mathbb{E} [\cL(\theta_0)]\nonumber \\
        &- \sum_{t=0}^{T-1} \frac{\EE\|\nabla_\theta \cL(\theta_t)\|_2^4}{2L\left( \EE\|\nabla_\theta \cL(\theta_t)\|_2^2 + \frac{1}{N_t} \text{tr}(\bH(\theta_t)) \right)} .
    \end{align}
\end{thm}

\begin{thm}[Convergence Analysis]
    Under the optimal learning rate schedule~\ref{thm:opt_lr}, we have
    \begin{equation}
        \frac{1}{T}\sum_{t=0}^{T-1} \EE\|\nabla_\theta \cL(\theta_t)\|_2^2 =
        O\left(\frac{1}{\sqrt{T}}\right),
    \end{equation}
    where the big-$O$ notation hides constants and other problem-dependent parameters independent of $T$.
\end{thm}
The proofs can be found in Appendix~\ref{app:A.3}.

\subsection{Optimal Baseline Design}
\label{sec:4.4}
In this section, we minimize $\tr(\bH(\theta))$, the total gradient variance appearing in Theorem~\ref{thm:upperbound}.
\begin{thm}[KL-Regularized Optimal Baseline]
\label{thm:optimal_beta_baseline}
The variance-optimal scalar baseline is
\begin{align}
b_\theta(q)
=\frac{\EE_{o\sim\pi_\theta(\cdot|q)}
\|\nabla_\theta\log\pi_\theta(o|q)\|_2^2
\left(F(q,o)-\beta\log
\frac{\pi_\theta(o|q)}{\pi_{\mathrm{ref}}(o|q)}\right)}{\EE_{o\sim\pi_\theta(\cdot|q)}
[\|\nabla_\theta\log\pi_\theta(o|q)\|_2^2]}.
\label{eq:forward_beta_baseline}
\end{align}
\end{thm}
Thus, the original formula is preserved, with the reward replaced by its KL-regularized counterpart. When $\beta=0$, Theorem~\ref{thm:optimal_beta_baseline} exactly recovers the original gradient-weighted baseline.

\begin{takeaway}{2}
Under KL regularization, the baseline should be a gradient-weighted average of the KL-regularized reward, rather than a simple reward average.
\end{takeaway}

The proof can be found in Appendix~\ref{app:A.4}.

\paragraph{Variance Reduction}
The optimal baseline minimizes the total variance $\bH(\theta)$, including the variance and cross-covariance induced by the KL-gradient estimator. This yields the smallest variance term in the loss upper bound and supports more stable policy updates.

\begin{algorithm}[t]
\caption{Optimal Baseline and Learning-Rate Policy Optimization (OBLR-PO). \baselinecolor{Blue: variance-optimal baseline}; \lrcolor{orange: SNR-adaptive learning rate}; black: shared operations.}
\label{alg:OBLRPO}
\small
\begin{algorithmic}[1]
\REQUIRE Initial policy $\pi_{\theta_0}$, reference policy $\pi_{\mathrm{ref}}$, reward function $F$, KL coefficient $\beta$, base rate $\eta_0$, and rate band $[\eta_{\min},\eta_{\max}]$
\FOR{$t=0,\dots,T-1$}
    \STATE Set $\theta_{\mathrm{old}}\leftarrow\theta_t$ and sample $N_t$ queries
    \STATE For each $q$, sample $G_t$ responses $o_i\sim\pi_{\theta_{\mathrm{old}}}(\cdot|q)$
    \FOR{each sampled response $o_i$}
        \STATE \baselinecolor{Compute $R_i=F(q,o_i)-\beta\log\frac{\pi_{\theta_t}(o_i|q)}{\pi_{\mathrm{ref}}(o_i|q)}$}
        \STATE \baselinecolor{Compute $s_i=\|\nabla_\theta\log\pi_{\theta_t}(o_i|q)\|_2^2$}
    \ENDFOR
    \FOR{each sampled response $o_i$}
        \STATE \baselinecolor{Set $\hat b_i=\frac{\sum_{j\ne i}s_jR_j}{\sum_{j\ne i}s_j}$ and $\hat A_i=R_i-\hat b_i$}
    \ENDFOR
    \STATE Construct importance-weighted gradient samples $\{g_{i,q}\}$ using $\hat A_i$
    \STATE Aggregate $\hat g_t=\frac{1}{N_tG_t}\sum_{q,i}g_{i,q}$
    \STATE \lrcolor{Estimate $\widehat{\operatorname{SNR}}(\theta_t)$ from the gradient samples}
    \STATE \lrcolor{Set $\hat\eta_t=\operatorname{clamp}\!\left(\eta_0\frac{N_t\widehat{\operatorname{SNR}}(\theta_t)}{1+N_t\widehat{\operatorname{SNR}}(\theta_t)},\eta_{\min},\eta_{\max}\right)$}
    \STATE Update $\theta_{t+1}\leftarrow\theta_t+\hat\eta_t\hat g_t$
\ENDFOR
\end{algorithmic}
\end{algorithm}

\section{Methodology}
\label{sec:5}
Our theory motivates two complementary components: a score-norm-weighted baseline that reduces gradient variance and an SNR-adaptive learning rate that scales updates by gradient reliability. Either component can be used independently; together they form \emph{Optimal Baseline and Learning-Rate Policy Optimization (OBLR-PO)}.

\subsection{Variance-Optimal Baseline}
\label{sec:5.baseline}
At iteration $t$, the old policy $\pi_{\theta_{\mathrm{old}}}$ samples $G_t$ responses per query. For each response $o_i$, we define the KL-regularized reward
\begin{equation}
R_i=F(q,o_i)-\beta\log
\frac{\pi_\theta(o_i|q)}{\pi_{\mathrm{ref}}(o_i|q)}.
\label{eq:empirical_kl_reward}
\end{equation}
Let $s_i=\|\nabla_\theta\log\pi_\theta(o_i|q)\|_2^2$ be the squared score norm. Following Theorem~\ref{thm:optimal_beta_baseline}, we use the score-norm-weighted leave-one-out baseline
\begin{equation}
\hat b_\theta(q,o_i)
=\frac{\sum_{j\ne i}s_jR_j}{\sum_{j\ne i}s_j},
\label{eq:empirical_beta_baseline}
\end{equation}
and advantage
\begin{equation}
\hat A(q,o_i)=R_i-\hat b_\theta(q,o_i).
\label{eq:empirical_optimal_advantage}
\end{equation}
The leave-one-out construction avoids coupling a sample to its own baseline, while score-norm weighting emphasizes responses that contribute more to gradient variance. When all score norms are equal, it reduces to the standard leave-one-out reward mean. The baseline-only variant uses these advantages with a fixed learning rate (Algorithm~\ref{alg:optimal_baseline} in Appendix~\ref{app:method_algorithms}).

\subsection{SNR-Adaptive Learning Rate}
\label{sec:5.learning_rate}
Given gradient samples from any policy-gradient estimator, we estimate $\widehat{\operatorname{SNR}}(\theta_t)$ and use the practical counterpart of Theorem~\ref{thm:opt_lr},
\begin{equation}
\tilde\eta_t=\eta_0\frac{N_t\widehat{\operatorname{SNR}}(\theta_t)}
{1+N_t\widehat{\operatorname{SNR}}(\theta_t)},
\qquad
\hat\eta_t=\operatorname{clamp}(\tilde\eta_t,\eta_{\min},\eta_{\max}).
\label{eq:empirical_adaptive_lr}
\end{equation}
Here, $\eta_0$ absorbs the unknown smoothness factor, and clamping prevents extreme updates. The step size increases with SNR, assigning conservative updates to noisy gradients and approaching the base rate for reliable gradients. Because the rule is baseline-agnostic, it can augment estimators such as RLOO, ReMax, and GRPO. The learning-rate-only variant applies this rule with the chosen estimator's original baseline (Algorithm~\ref{alg:adaptive_lr} in Appendix~\ref{app:method_algorithms}).

\subsection{Full OBLR-PO Update}
\label{sec:5.oblrpo}
OBLR-PO first constructs KL-regularized, score-norm-weighted leave-one-out advantages, then uses the resulting gradient samples to estimate the SNR and update the policy. The baseline controls the variance of the update direction, while the adaptive rule controls its magnitude. Algorithm~\ref{alg:OBLRPO} summarizes the procedure: the \baselinecolor{blue steps} implement the baseline, the \lrcolor{orange steps} implement learning-rate adaptation, and the black steps are shared. Using only the blue steps with a fixed rate gives the baseline-only variant; using only the orange steps with a selected base-method baseline gives the learning-rate-only variant. Appendix~\ref{app:estimators} provides the importance-weighted gradient and micro-batch SNR estimators.

\section{Experiments}
\label{sec:6}

\begin{figure*}[t]
\centering
\begin{subfigure}{0.32\textwidth}
    \includegraphics[width=\linewidth]{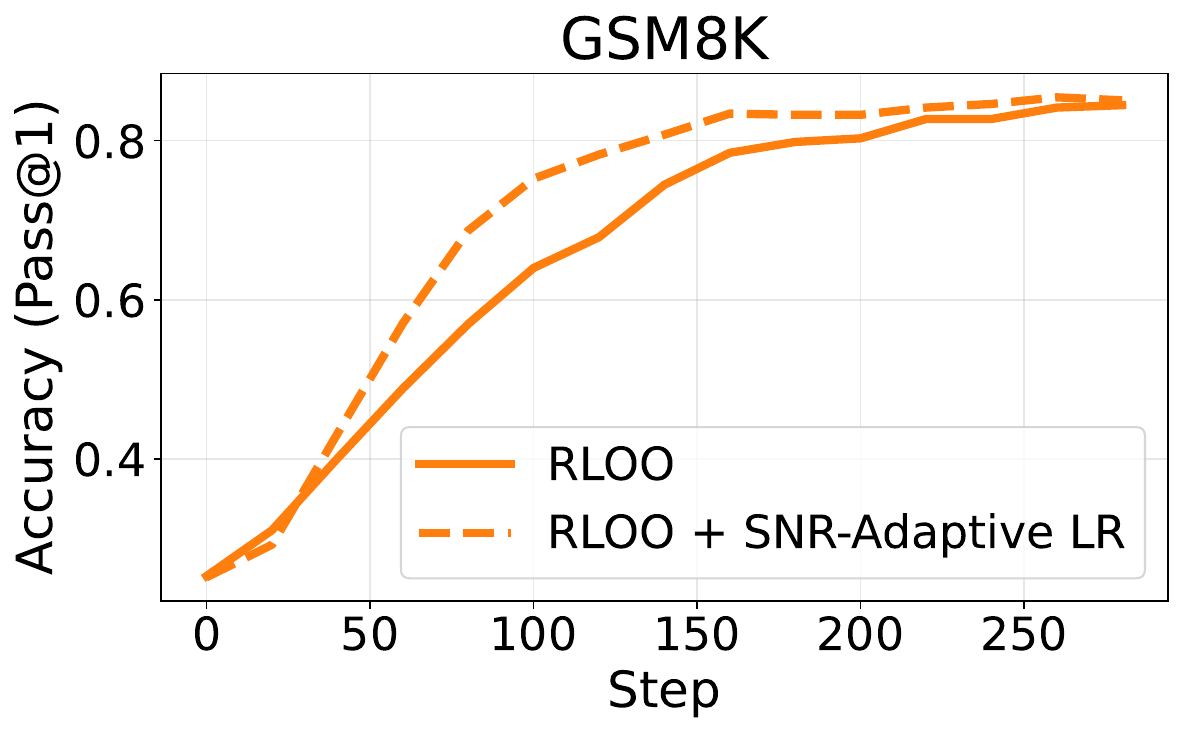}
    \caption{RLOO, accuracy}
\end{subfigure}
\hfill
\begin{subfigure}{0.32\textwidth}
    \includegraphics[width=\linewidth]{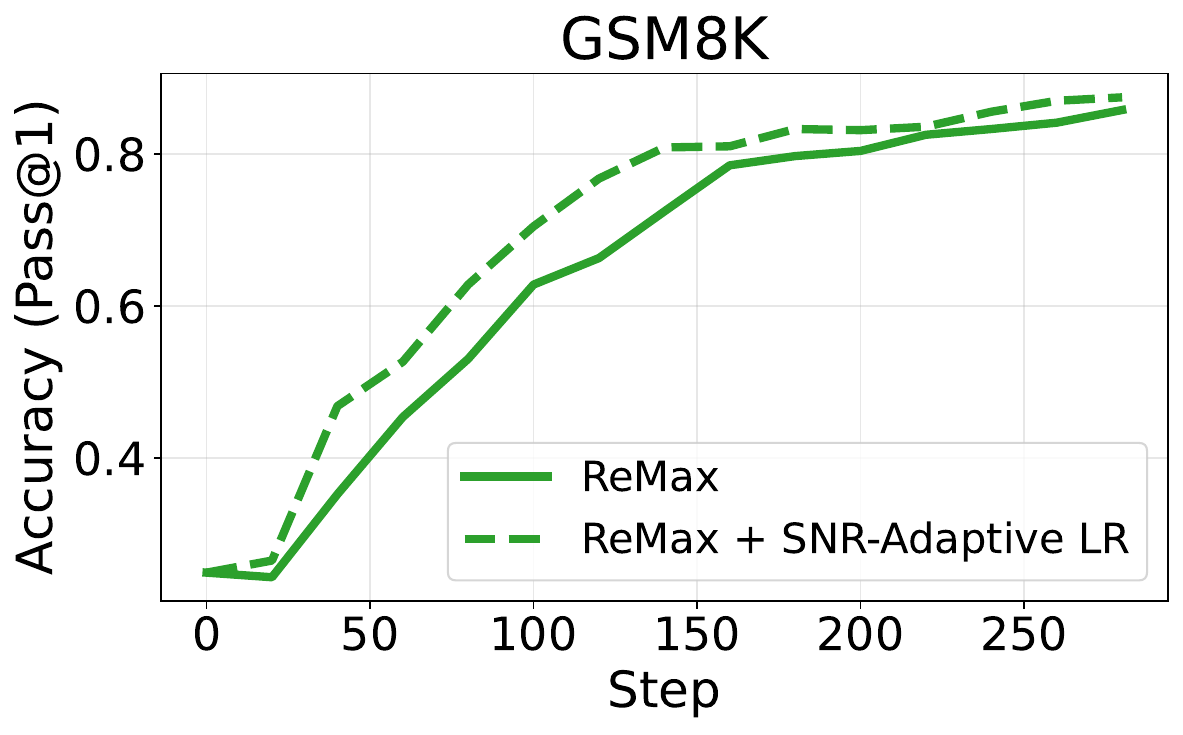}
    \caption{ReMax, accuracy}
\end{subfigure}
\hfill
\begin{subfigure}{0.32\textwidth}
    \includegraphics[width=\linewidth]{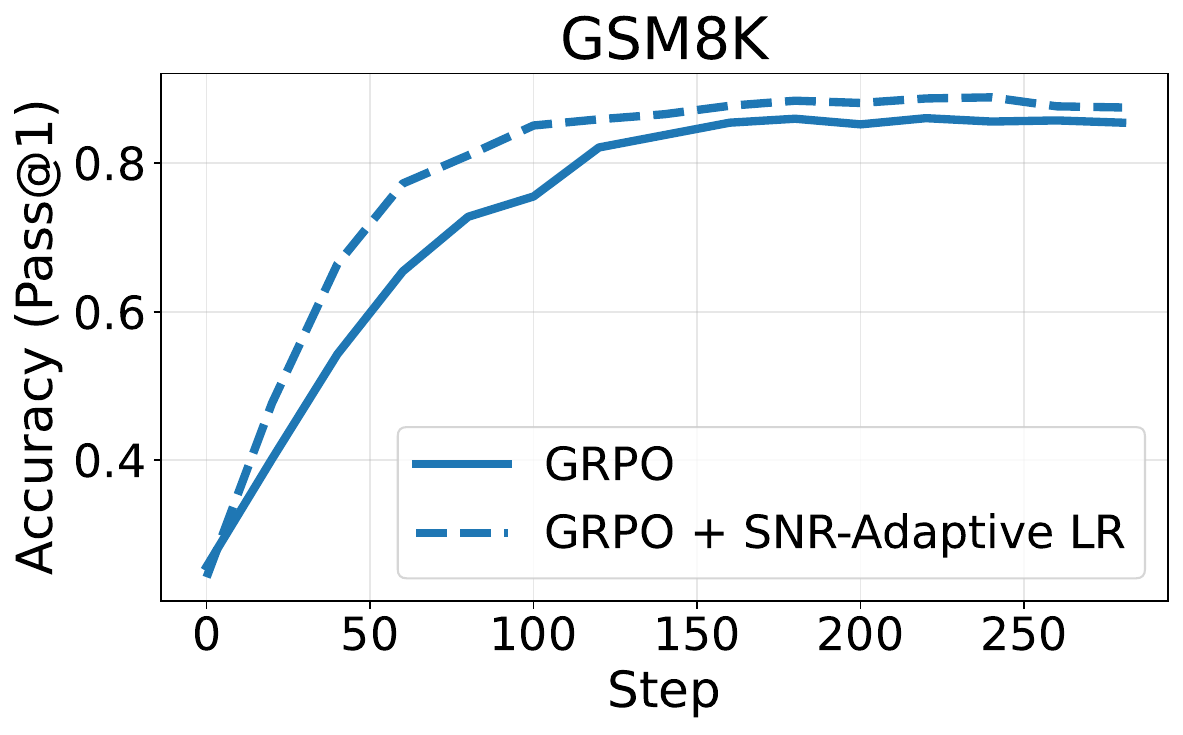}
    \caption{GRPO, accuracy}
\end{subfigure}

\vspace{0.4em}
\begin{subfigure}{0.32\textwidth}
    \includegraphics[width=\linewidth]{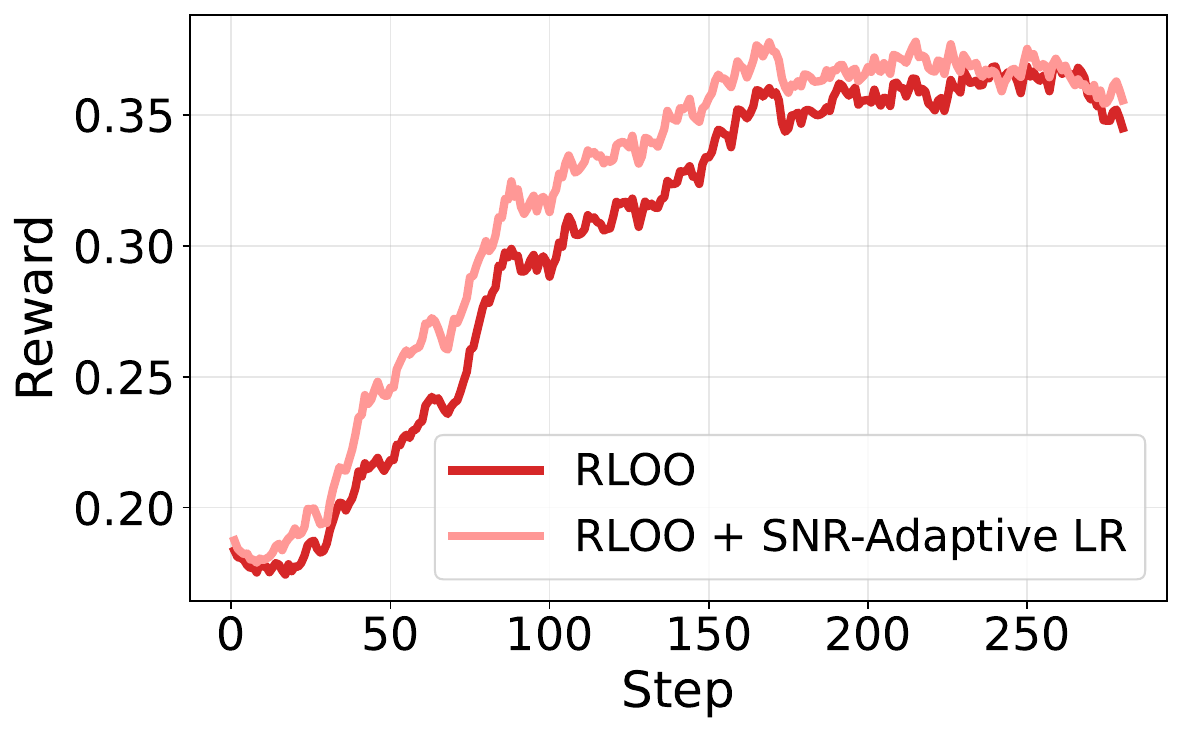}
    \caption{RLOO, reward}
\end{subfigure}
\hfill
\begin{subfigure}{0.32\textwidth}
    \includegraphics[width=\linewidth]{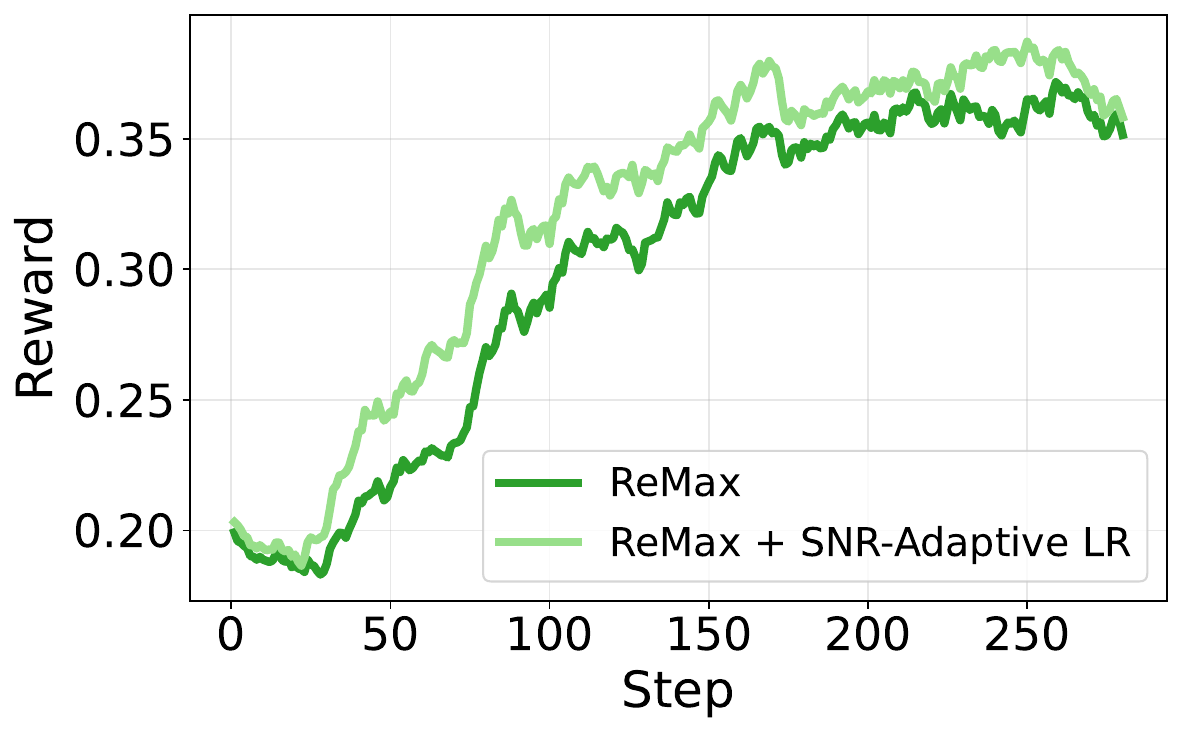}
    \caption{ReMax, reward}
\end{subfigure}
\hfill
\begin{subfigure}{0.32\textwidth}
    \includegraphics[width=\linewidth]{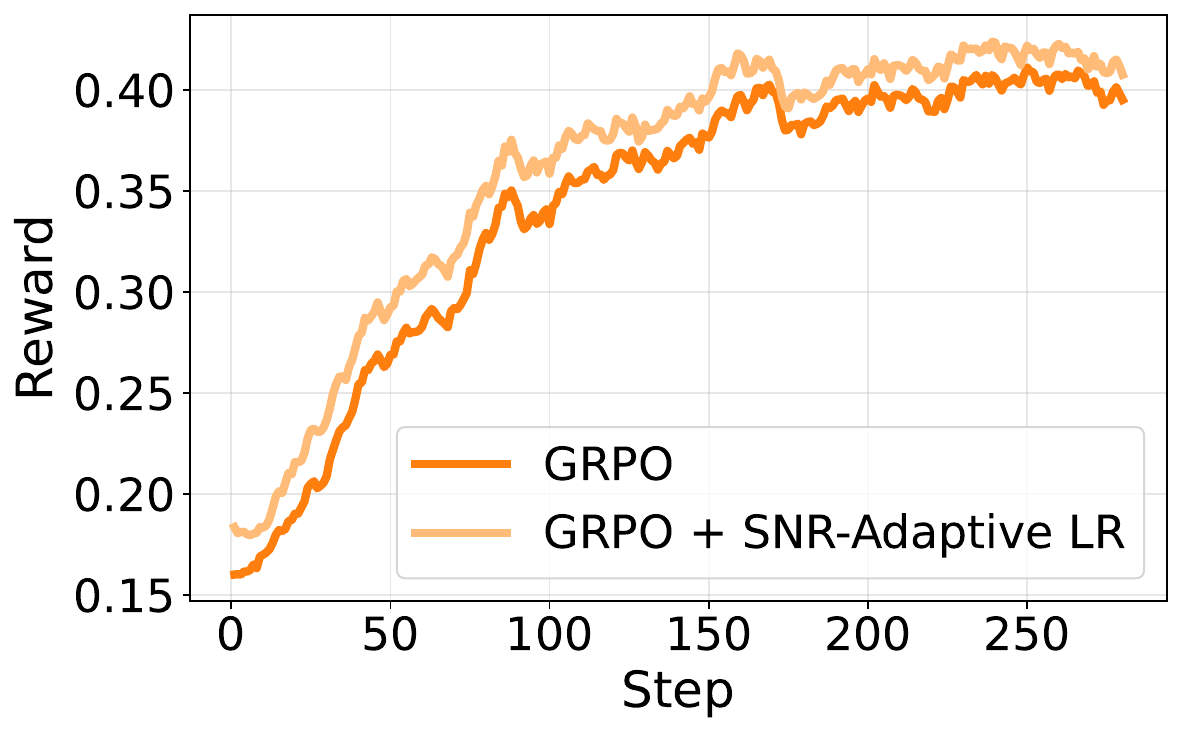}
    \caption{GRPO, reward}
\end{subfigure}
\caption{Effect of the SNR-Adaptive LR versus a fixed learning rate, for RLOO, ReMax, and GRPO. Top row: Pass@1 accuracy on GSM8K; bottom row: training reward.}
\label{fig:lr}
\end{figure*}

We empirically validate the two components of our framework---the SNR-adaptive learning rate (Algorithm~\ref{alg:adaptive_lr}) and the variance-optimal baseline (Algorithm~\ref{alg:optimal_baseline})---as well as their combination, the full OBLR-PO method. Our experiments are organized into three parts. We first show that the SNR-adaptive learning rate improves any base policy-gradient estimator (Section~\ref{sec:6.2}). We then show that the variance-optimal baseline reduces gradient variance and improves accuracy relative to standard baselines (Section~\ref{sec:6.3}). Finally, we show that the full OBLR-PO method outperforms existing policy optimization algorithms (Section~\ref{sec:6.4}).

\begin{table*}[t]
\centering
\caption{Final Pass@1 accuracy (\%, $\uparrow$) across four benchmarks, and the gradient noise-to-signal ratio $1/\widehat{\mathrm{SNR}}$ ($\times10^4$, $\downarrow$) of each estimator. \emph{Variance-Optimal Baseline} uses a fixed learning rate; \emph{OBLR-PO} adds the SNR-adaptive schedule. Best per column in \textbf{bold}, second best \underline{underlined}.}
\label{tab:main}
\resizebox{0.9\textwidth}{!}{%
\begin{tabular}{lccccc}
\toprule
 & \multicolumn{4}{c}{Accuracy (Pass@1, \%)\,$\uparrow$} & $1/\widehat{\mathrm{SNR}}$ ($\times10^4$)\,$\downarrow$ \\
\cmidrule(lr){2-5}\cmidrule(lr){6-6}
Method & AMC23 & GSM8K & MATH-500 & OlympiadBench & \\
\midrule
PPO~\cite{schulman2018highdimensionalcontinuouscontrolusing}    & $52.50$ & $81.05$ & $69.40$ & $24.11$ & $1.74$ \\
ReMax~\cite{li2024remaxsimpleeffectiveefficient}  & $45.00$ & $82.41$ & $70.80$ & $\underline{24.55}$ & $1.63$ \\
RLOO~\cite{ahmadian2024basicsrevisitingreinforcestyle}   & $47.50$ & $83.93$ & $70.60$ & $23.81$ & $1.67$ \\
GRPO~\cite{shao2024deepseekmathpushinglimitsmathematical}   & $52.50$ & $85.60$ & $70.20$ & $23.07$ & $\underline{1.43}$ \\
\midrule
Variance-Optimal Baseline (Ours) & $\underline{55.00}$ & $\underline{85.90}$ & $\underline{71.80}$ & $\mathbf{25.74}$ & $\mathbf{1.34}$ \\
OBLR-PO (Ours)          & $\mathbf{57.50}$ & $\mathbf{87.19}$ & $\mathbf{72.20}$ & $\underline{24.55}$ & -- \\
\bottomrule
\end{tabular}%
}
\end{table*}

\subsection{Experimental Setup}
\label{sec:6.1}
All experiments post-train \textbf{Qwen3-4B-Base} on the \textbf{DeepMath-103K}~\cite{he2025deepmath} dataset with verifiable rule-based rewards, optimized with SGD using verl~\cite{sheng2024hybridflow} framework. We evaluate on four mathematical reasoning benchmarks---\textbf{AMC23}~\cite{mathai_amc23}, \textbf{GSM8K}~\cite{cobbe2021trainingverifierssolvemath}, \textbf{MATH-500}~\cite{hendrycksmath2021, lightman2023lets}, and \textbf{OlympiadBench}~\cite{he2024olympiadbench}---reporting accuracy (Pass@1). All runs train for a maximum of $285$ steps on 4 A800 GPUs. The full hyperparameter configuration, including the SNR-adaptive learning-rate band, is provided in Appendix~\ref{app:setup}.

\subsection{Effectiveness of the SNR-Adaptive Learning Rate}
\label{sec:6.2}

We first isolate the SNR-adaptive learning rate (Algorithm~\ref{alg:adaptive_lr}), which can be applied on top of any policy-gradient estimator while leaving its baseline unchanged. For each base method, we compare a fixed learning rate against the SNR-adaptive learning rate, keeping all other settings identical.

Figure~\ref{fig:lr} reports representative Pass@1 accuracy curves on GSM8K (top row) together with the corresponding training-reward curves (bottom row), for RLOO, ReMax, and GRPO. Adding the SNR-Adaptive LR consistently matches or improves both accuracy and reward, confirming Theorem~\ref{thm:opt_lr}. The learning-rate rule thus behaves as a lightweight, drop-in improvement rather than a method-specific tweak, adding negligible wall-clock overhead (Appendix~\ref{app:experiments}). The same trend holds on the remaining benchmarks; the accuracy curves for MATH-500 and OlympiadBench are deferred to Appendix~\ref{app:experiments}.

We highlight that RLOO and ReMax satisfy Assumption~\ref{ass:1}, and are therefore directly covered by our theory. GRPO, by contrast, does \emph{not} satisfy Assumption~\ref{ass:1}, so the GRPO${+}$SNR-Adaptive LR curves are reported as an \emph{additional, beyond-theory} experiment. Notably, the SNR-adaptive learning rate still delivers clear gains for GRPO, suggesting that the mechanism is robust even outside the regime.

\begin{figure}[t]
\centering
\begin{subfigure}{0.275\linewidth}
    \includegraphics[width=\linewidth]{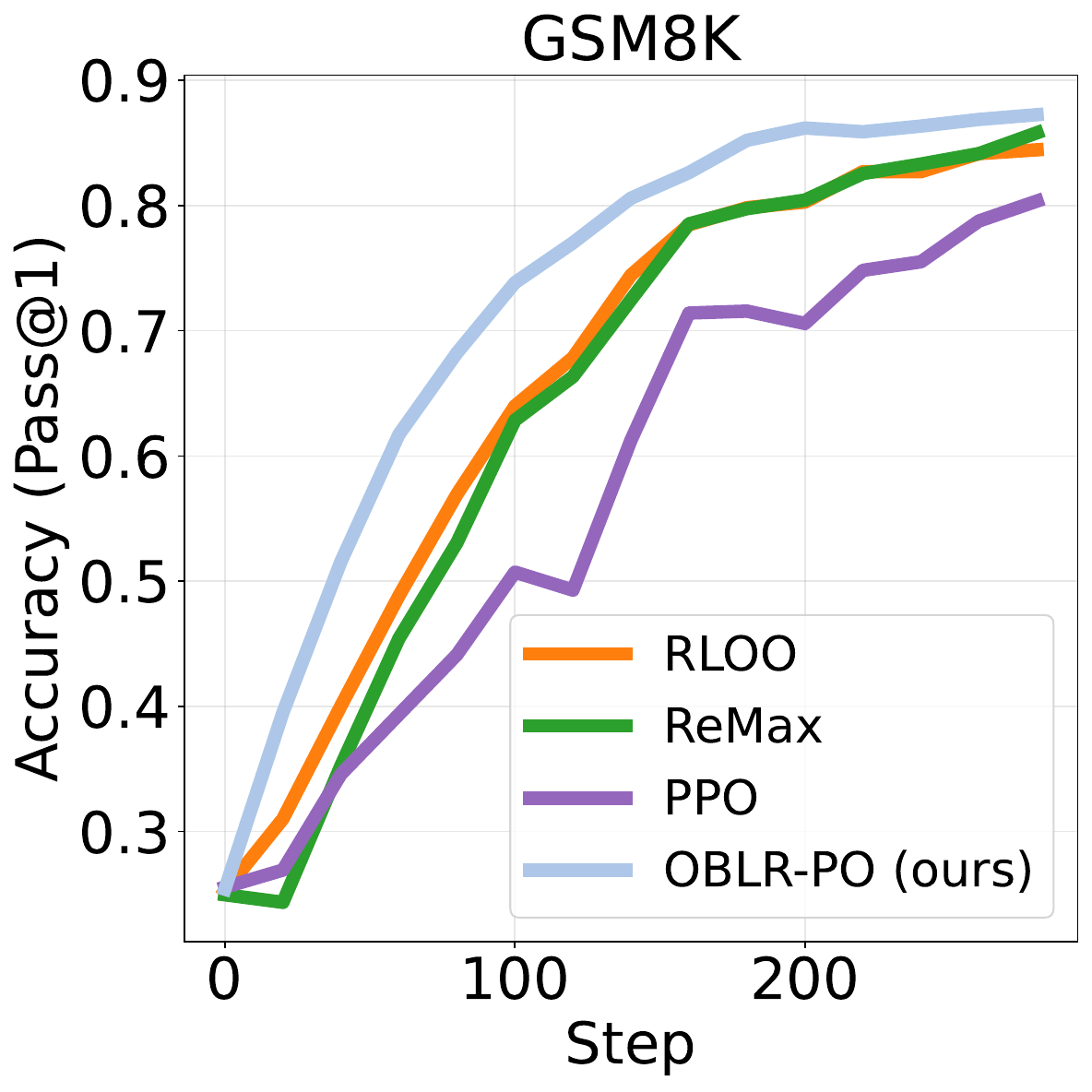}
    \caption{GSM8K accuracy}
\end{subfigure}
\hfill
\begin{subfigure}{0.275\linewidth}
    \includegraphics[width=\linewidth]{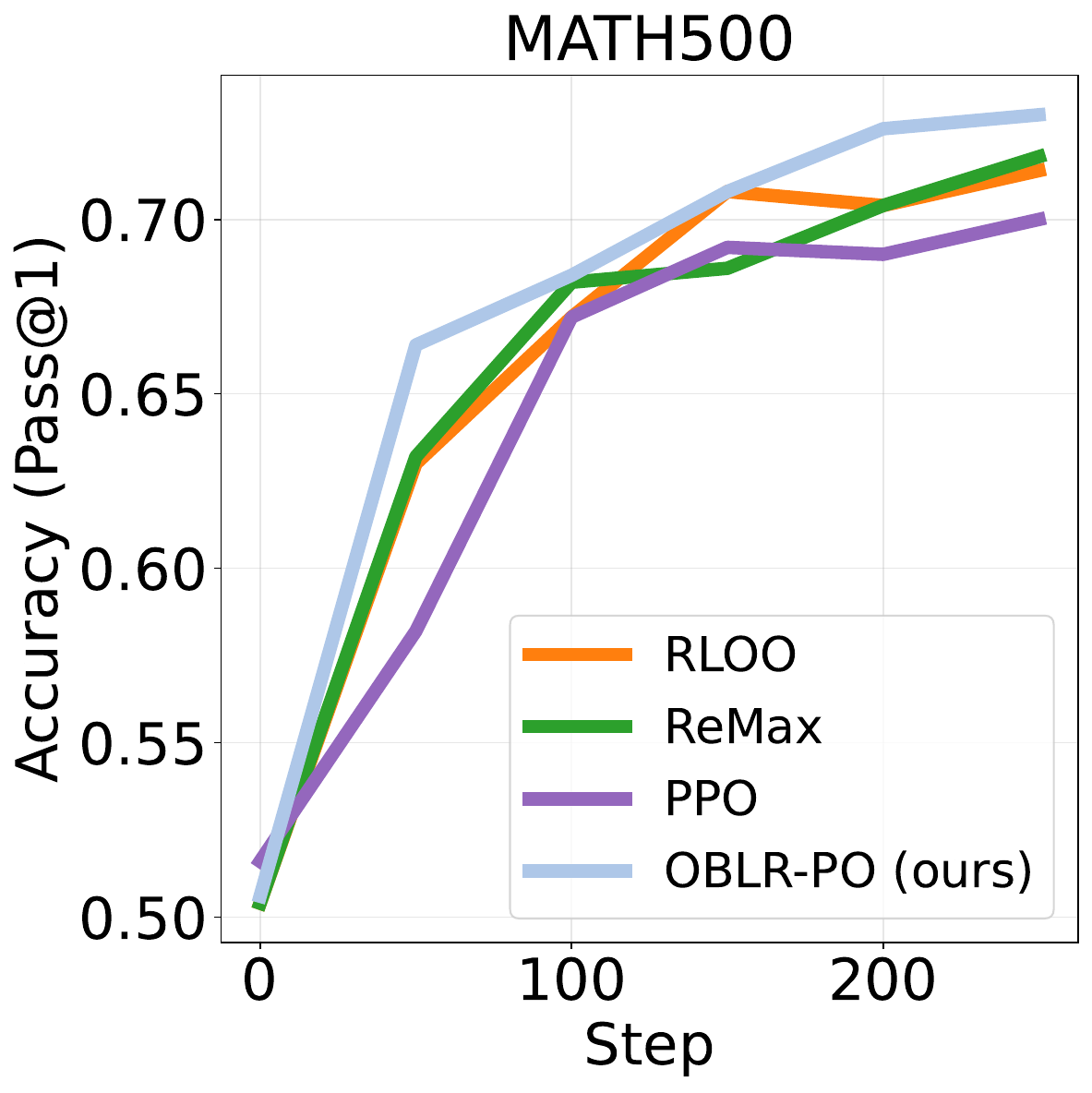}
    \caption{MATH500 accuracy}
\end{subfigure}
\hfill
\begin{subfigure}{0.42\linewidth}
    \includegraphics[width=\linewidth]{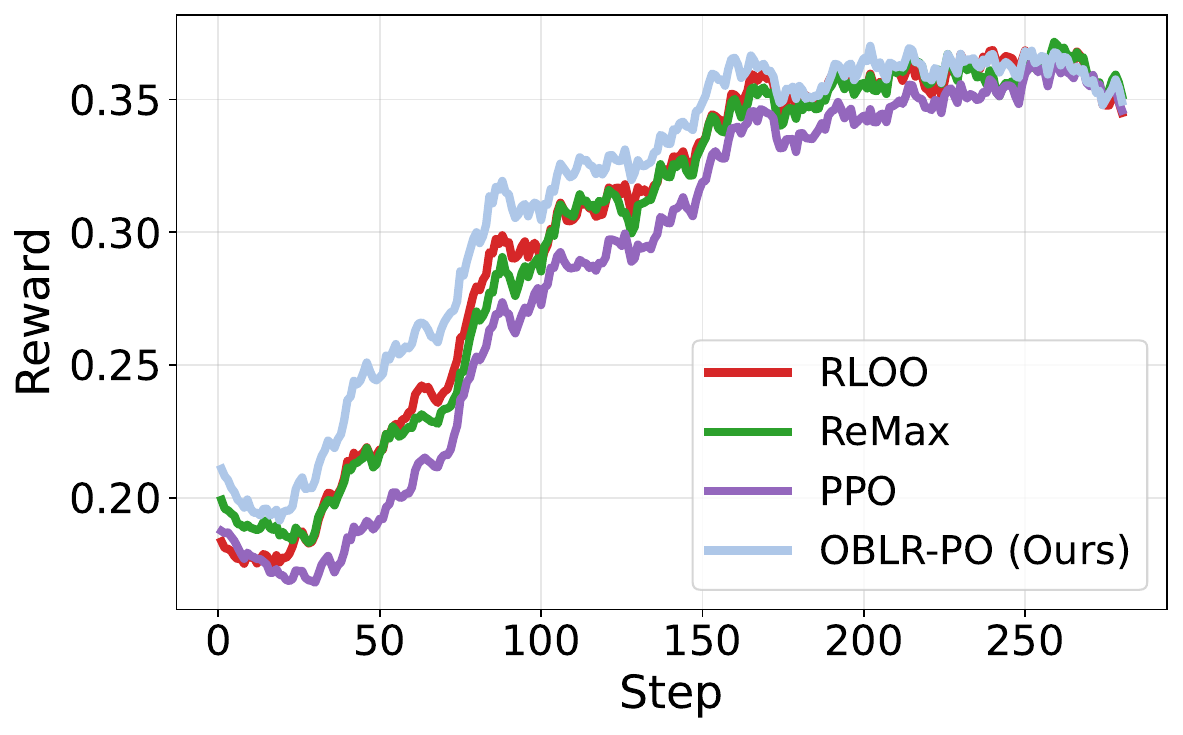}
    \caption{Training reward}
\end{subfigure}

\caption{Performance comparison of full OBLR-PO against RLOO, ReMax, and PPO. The left and middle panels show Pass@1 accuracy trajectories on GSM8K and MATH500, respectively. The right panel presents training reward curves.}
\label{fig:main-compare}
\vspace{-5mm}
\end{figure}

\vspace{-1mm}
\subsection{Effectiveness of the Variance-Optimal Baseline}
\label{sec:6.3}
\vspace{-1mm}

We next isolate the variance-optimal baseline (Algorithm~\ref{alg:optimal_baseline}), which uses the gradient-weighted, KL-regularized leave-one-out estimator of Equation~\eqref{eq:empirical_beta_baseline}. Following Section~\ref{sec:5.baseline}, this component is evaluated at a \emph{fixed} learning rate so that it can be compared directly against standard baselines without confounding from learning-rate adaptation. 

\paragraph{Final accuracy.} Table~\ref{tab:main} reports the final Pass@1 across the four benchmarks. Evaluated at a fixed learning rate, the variance-optimal baseline attains the best accuracy on \emph{all four} benchmarks among the fixed-learning-rate methods---PPO, ReMax, RLOO, and GRPO---confirming that the gradient-weighted, KL-regularized baseline yields stronger final performance than standard reward-mean, leave-one-out, and value-based baselines.

\paragraph{Gradient variance.} The last column of Table~\ref{tab:main} reports the inverse signal-to-noise ratio $1\over\widehat{\mathrm{SNR}}$ of the policy gradient (0.5\% trimmed mean over training). We evaluate the \emph{normalized} noise $1\over\widehat{\mathrm{SNR}}$ rather than the raw gradient-noise trace $\tr(\bH(\theta))$ because the raw variance is not comparable across estimators due to the different overall magnitude of different gradients. Dividing the noise by the squared signal removes this scale dependence and yields a dimensionless, scale-invariant measure. The variance-optimal baseline attains the lowest $1\over\widehat{\mathrm{SNR}}$ ($1.34\times10^4$) among all estimators, empirically confirming that the gradient-weighted baseline of Theorem~\ref{thm:optimal_beta_baseline} maximizes the gradient SNR (Takeaway~2, Section~\ref{sec:4.4}).

\subsection{Effectiveness of OBLR-PO}
\label{sec:6.4}
Finally, we evaluate the full OBLR-PO method, which combines the variance-optimal baseline with the SNR-adaptive learning rate. As shown in Table~\ref{tab:main}, adding the SNR-adaptive learning rate on top of the variance-optimal baseline further improves accuracy on AMC23, GSM8K, and MATH-500, yielding the best overall results on these benchmarks; on OlympiadBench the variance-optimal baseline alone remains highest, indicating a mild trade-off from the adaptive schedule on that set. Figures~\ref{fig:main-compare} corroborate this on the training trajectories, where OBLR-PO attains the strongest accuracy and reward curves against RLOO, ReMax, and PPO.

The learning-rate ablation in Appendix~\ref{app:experiments} (Figure~\ref{fig:app-oblrpo-incr}) isolates this effect, contrasting the variance-optimal baseline alone against the full OBLR-PO and showing a further gain in both accuracy and reward---closing the loop on the two-component design: the learning-rate rule (Section~\ref{sec:6.2}) and the baseline (Section~\ref{sec:6.3}) each help in isolation, and combine into a stronger method.

\section{Conclusion and Limitation}
In this work, we developed a theoretical framework that rigorously characterizes the bias, variance, and convergence of policy optimization under mild assumptions. Our analysis establishes an optimal learning-rate schedule governed by the signal-to-noise ratio and identifies the gradient-weighted baseline as a principled solution for variance reduction. These two improvements can be applied independently or combined as the full \emph{OBLR-PO} method. Experiments show that each improves performance individually, while their combination provides further gains. The learning-rate rule can also be naturally integrated with existing policy optimization methods to improve their performance.

However, two limitations remain. First, our guarantees are given with respect to an upper bound on the loss, leaving a gap to the realized optimization dynamics. Second, the $L$-smoothness assumption (Assumption~\ref{ass:2}), while common in theory, may not strictly hold in practice and requires further empirical validation. We hope these findings motivate future work to tighten theoretical bounds and test assumptions in large-scale RL for LLMs.

\section*{Acknowledgments}
The authors would like to express their sincere gratitude to Alibaba Group US DAMO Academy for its generous support throughout the course of this research and experimental process.

\bibliographystyle{plainnat}
\bibliography{reference}

\appendix
\onecolumn
\hrule height 4pt
\vskip 0.25in
\vskip -\parskip
\vbox{
    \centering
    \LARGE 
    \textbf{Appendix}
}
\vskip 0.29in
\vskip -\parskip
\hrule height 1pt
\renewcommand*\footnoterule{} 

\newcommand\blfootnote[1]{%
  \begingroup
  \renewcommand\thefootnote{}\footnote{#1}%
  \addtocounter{footnote}{-1}%
  \endgroup
}
\renewcommand{\thethm}{A.\arabic{thm}}
\setcounter{thm}{0}


\section{Proofs for Section~\ref{sec:theory}}
\label{app:A}

\subsection{Proofs for Section~\ref{sec:4.1}}
\label{app:A.1}

\begin{thm}[Unbiasedness]
The estimator in Equation~\eqref{eq:beta_gradient_estimator} satisfies
\begin{equation}
\EE\!\left[\widehat{\nabla_\theta J(\theta)}\right]
=\nabla_\theta J(\theta).
\end{equation}
\end{thm}
\begin{proof}
By the definition of the gradient estimator,
\begin{align}
\widehat{\nabla_\theta J(\theta)}
={}&\frac{1}{N_tG_t}\sum_{j=1}^{N_t}\sum_{i=1}^{G_t}
\nabla_\theta\log\pi_\theta(o_{i,j}|q_j)\nonumber\\
&\quad\cdot\left(F(q_j,o_{i,j})-b_\theta(q_j)
-\beta\log\frac{\pi_\theta(o_{i,j}|q_j)}
{\pi_{\mathrm{ref}}(o_{i,j}|q_j)}\right).
\label{eq:app_unbiased_start}
\end{align}
Since the query groups are identically distributed, linearity of
expectation gives
\begin{equation}
\EE\!\left[\widehat{\nabla_\theta J(\theta)}\right]=\EE_{q\sim D,\,o\sim\pi_\theta(\cdot|q)}\!\left[\nabla_\theta\log\pi_\theta(o|q)\left(F(q,o)-b_\theta(q)-\beta\log\frac{\pi_\theta(o|q)}{\pi_{\mathrm{ref}}(o|q)}\right)\right].
\label{eq:app_unbiased_expectation}
\end{equation}
We evaluate the three terms in
Equation~\eqref{eq:app_unbiased_expectation} separately. First, for each
fixed query $q$, Assumption~\ref{ass:1} implies that $b_\theta(q)$ is
independent of the sampled output $o$. Therefore,
\begin{equation}
\EE_{o\sim\pi_\theta(\cdot|q)}\!\left[b_\theta(q)\nabla_\theta\log\pi_\theta(o|q)\right]=b_\theta(q)\int\nabla_\theta\pi_\theta(o|q)\,\dd o=b_\theta(q)\nabla_\theta\int\pi_\theta(o|q)\,\dd o=0.
\label{eq:app_baseline_zero}
\end{equation}
Second, the score-function identity yields
\begin{equation}
\EE_{o\sim\pi_\theta(\cdot|q)}\!\left[F(q,o)\nabla_\theta\log\pi_\theta(o|q)\right]=\int F(q,o)\nabla_\theta\pi_\theta(o|q)\,\dd o=\nabla_\theta\EE_{o\sim\pi_\theta(\cdot|q)}[F(q,o)].
\label{eq:app_reward_gradient}
\end{equation}
Third, because $\pi_{\mathrm{ref}}$ does not depend on $\theta$,
\begin{align}
&\nabla_\theta D_{\mathrm{KL}}
\left(\pi_\theta(\cdot|q)\|\pi_{\mathrm{ref}}(\cdot|q)\right)\nonumber\\
&=\nabla_\theta\int\pi_\theta(o|q)
\log\frac{\pi_\theta(o|q)}{\pi_{\mathrm{ref}}(o|q)}\,\dd o\nonumber\\
&=\EE_{o\sim\pi_\theta(\cdot|q)}\!\left[
\left(1+\log\frac{\pi_\theta(o|q)}{\pi_{\mathrm{ref}}(o|q)}\right)
\nabla_\theta\log\pi_\theta(o|q)\right]\nonumber\\
&=\EE_{o\sim\pi_\theta(\cdot|q)}\!\left[
\log\frac{\pi_\theta(o|q)}{\pi_{\mathrm{ref}}(o|q)}
\nabla_\theta\log\pi_\theta(o|q)\right],
\label{eq:app_kl_gradient}
\end{align}
where the last equality again follows from
$\EE_{o\sim\pi_\theta(\cdot|q)}[\nabla_\theta\log\pi_\theta(o|q)]=0$.
Substituting Equations~\eqref{eq:app_baseline_zero}--
\eqref{eq:app_kl_gradient} into
Equation~\eqref{eq:app_unbiased_expectation}, and then averaging over
$q\sim D$, proves
\begin{equation}
\EE\!\left[\widehat{\nabla_\theta J(\theta)}\right]
=\nabla_\theta J(\theta).
\end{equation}
\end{proof}

\begin{thm}[Variance Expression]
The covariance matrix of $\xi(\theta)$ is
\begin{equation}
\Var[\xi(\theta)]
=\frac{1}{N_tG_t}\bH(\theta)
+\frac{G_t-1}{N_tG_t}\bC(\theta).
\end{equation}
\end{thm}
\begin{proof}
Recall that
\begin{equation}
\xi(\theta)=\widehat{\nabla_\theta J(\theta)}-\nabla_\theta J(\theta).
\end{equation}
Subtracting the deterministic mean does not change covariance. Hence,
using Equation~\eqref{eq:beta_gradient_estimator},
\begin{align}
\Var[\xi(\theta)]
={}&\Var\!\Bigg[\frac{1}{N_tG_t}
\sum_{j=1}^{N_t}\sum_{i=1}^{G_t}
\nabla_\theta\log\pi_\theta(o_{i,j}|q_j)\left(F(q_j,o_{i,j})-b_\theta(q_j)
-\beta\log\frac{\pi_\theta(o_{i,j}|q_j)}
{\pi_{\mathrm{ref}}(o_{i,j}|q_j)}\right)\Bigg].
\label{eq:app_variance_start}
\end{align}
The $N_t$ query groups are independent. Thus, the covariance of the sum
over queries is the sum of the $N_t$ within-query covariance matrices:
\begin{align}
\Var[\xi(\theta)]
={}&\frac{1}{N_t^2G_t^2}\sum_{j=1}^{N_t}
\Var\!\Bigg[\sum_{i=1}^{G_t}
\nabla_\theta\log\pi_\theta(o_{i,j}|q_j)\left(F(q_j,o_{i,j})-b_\theta(q_j)
-\beta\log\frac{\pi_\theta(o_{i,j}|q_j)}
{\pi_{\mathrm{ref}}(o_{i,j}|q_j)}\right)\Bigg].
\label{eq:app_variance_queries}
\end{align}
For a fixed query group, the covariance of the sum consists of $G_t$
diagonal variance terms and $G_t(G_t-1)$ ordered cross-covariance terms.
By exchangeability, every diagonal term is $\bH(\theta)$ and every
cross term is $\bC(\theta)$. Consequently,
\begin{equation}
\Var[\xi(\theta)]=\frac{N_t}{N_t^2G_t^2}\left(G_t\bH(\theta)+G_t(G_t-1)\bC(\theta)\right)=\frac{1}{N_tG_t}\bH(\theta)+\frac{G_t-1}{N_tG_t}\bC(\theta).
\end{equation}
This is the claimed variance expression.
\end{proof}

\subsection{Proofs for Section~\ref{sec:4.2}}
\label{app:A.2}

\begin{thm}[Upper Bound]
\label{thm:a.3}
Under Assumptions~\ref{ass:1}--\ref{ass:3} and Assumption~\ref{ass:4},
\begin{align}
\EE[\cL(\theta_T)]
\le{}&\EE[\cL(\theta_0)]
-\sum_{t=0}^{T-1}\eta_t\EE\|\nabla\cL(\theta_t)\|_2^2+\frac{L}{2}\sum_{t=0}^{T-1}\eta_t^2
\left(\EE\|\nabla\cL(\theta_t)\|_2^2
+\frac{1}{N_t}\tr(\bH(\theta_t))\right).
\label{eq:upper_bound}
\end{align}
\end{thm}
\begin{proof}
Since $\cL(\theta)=J(\theta^*)-J(\theta)$, we have
$\nabla\cL(\theta)=-\nabla J(\theta)$. The policy update can therefore
be written as
\begin{equation}
\theta_{t+1}
=\theta_t-\eta_t\bigl(\nabla\cL(\theta_t)-\xi(\theta_t)\bigr).
\label{eq:app_loss_update_ref}
\end{equation}
By Assumption~\ref{ass:2}, $J$ is $L$-smooth, and hence $\cL$ is also
$L$-smooth. Applying the smoothness inequality to
Equation~\eqref{eq:app_loss_update_ref} gives
\begin{align}
\cL(\theta_{t+1})
\le{}&\cL(\theta_t)
-\eta_t\left\langle\nabla\cL(\theta_t),
\nabla\cL(\theta_t)-\xi(\theta_t)\right\rangle+\frac{L\eta_t^2}{2}
\|\nabla\cL(\theta_t)-\xi(\theta_t)\|_2^2.
\label{eq:app_smoothness_step}
\end{align}
Taking expectation and using
$\EE[\xi(\theta_t)\mid\theta_t]=0$, the inner-product term becomes
\begin{align}
&\EE\!\left[\left\langle\nabla\cL(\theta_t),
\nabla\cL(\theta_t)-\xi(\theta_t)\right\rangle\right]
=\EE\|\nabla\cL(\theta_t)\|_2^2.
\end{align}
Similarly, expanding the squared norm gives
\begin{equation}
\EE\|\nabla\cL(\theta_t)-\xi(\theta_t)\|_2^2=\EE\|\nabla\cL(\theta_t)\|_2^2-2\EE\langle\nabla\cL(\theta_t),\xi(\theta_t)\rangle+\EE\|\xi(\theta_t)\|_2^2=\EE\|\nabla\cL(\theta_t)\|_2^2+\tr(\Var(\xi(\theta_t))).
\label{eq:app_expand_noise}
\end{equation}
It remains to control the last variance term. By the variance expression
in Appendix~\ref{app:A.1},
\begin{equation}
\tr(\Var(\xi(\theta_t)))
=\frac{1}{N_tG_t}\tr(\bH(\theta_t))
+\frac{G_t-1}{N_tG_t}\tr(\bC(\theta_t)).
\label{eq:app_variance_trace_ref}
\end{equation}
For each coordinate, Cauchy--Schwarz bounds the covariance between two
distinct, identically distributed within-query contributions by their
common variance. Summing over all coordinates gives
\begin{equation}
\tr(\bC(\theta_t))\le\tr(\bH(\theta_t)).
\end{equation}
Therefore,
\begin{equation}
\tr(\Var(\xi(\theta_t)))\le\frac{1}{N_tG_t}\tr(\bH(\theta_t))+\frac{G_t-1}{N_tG_t}\tr(\bH(\theta_t))=\frac{1}{N_t}\tr(\bH(\theta_t)).
\label{eq:app_variance_trace_bound_ref}
\end{equation}
Substituting Equations~\eqref{eq:app_expand_noise} and
\eqref{eq:app_variance_trace_bound_ref} into
Equation~\eqref{eq:app_smoothness_step}, we obtain
\begin{align}
\EE[\cL(\theta_{t+1})]
\le{}&\EE[\cL(\theta_t)]
-\eta_t\EE\|\nabla\cL(\theta_t)\|_2^2+\frac{L\eta_t^2}{2}
\left(\EE\|\nabla\cL(\theta_t)\|_2^2
+\frac{1}{N_t}\tr(\bH(\theta_t))\right).
\end{align}
Finally, summing this inequality from $t=0$ to $T-1$ telescopes the loss
terms and proves Equation~\eqref{eq:upper_bound}.
\end{proof}

\subsection{Proofs for Section~\ref{sec:4.3}}
\label{app:A.3}

\begin{thm}[Optimal Learning Rate Schedule]
The upper bound in Theorem~\ref{thm:upperbound} is minimized pointwise by
\begin{equation}
\eta_t=\frac{1}{L}
\frac{\EE\|\nabla\cL(\theta_t)\|_2^2}
{\EE\|\nabla\cL(\theta_t)\|_2^2+N_t^{-1}\tr(\bH(\theta_t))}.
\end{equation}
\end{thm}
\begin{proof}
From Equation~\eqref{eq:upper_bound}, all terms involving $\eta_t$ are
\begin{equation}
-\eta_t\EE\|\nabla\cL(\theta_t)\|_2^2+\frac{L}{2}\eta_t^2\left(\EE\|\nabla\cL(\theta_t)\|_2^2+\frac{1}{N_t}\tr(\bH(\theta_t))\right).
\label{eq:app_lr_ref_quadratic}
\end{equation}
This is a convex quadratic function of $\eta_t$. Its derivative is
\begin{equation}
-\EE\|\nabla\cL(\theta_t)\|_2^2+L\eta_t\left(\EE\|\nabla\cL(\theta_t)\|_2^2+\frac{1}{N_t}\tr(\bH(\theta_t))\right).
\end{equation}
Setting the derivative to zero and solving for $\eta_t$ gives
\begin{equation}
\eta_t=\frac{1}{L}
\frac{\EE\|\nabla\cL(\theta_t)\|_2^2}
{\EE\|\nabla\cL(\theta_t)\|_2^2+N_t^{-1}\tr(\bH(\theta_t))}.
\end{equation}
Using
$\operatorname{SNR}(\theta_t)=
\EE\|\nabla\cL(\theta_t)\|_2^2/\tr(\bH(\theta_t))$
immediately gives
\begin{equation}
\eta_t=\frac{1}{L}
\frac{N_t\operatorname{SNR}(\theta_t)}
{1+N_t\operatorname{SNR}(\theta_t)}.
\end{equation}
\end{proof}

\begin{thm}
Under the optimal schedule,
\begin{equation}
\EE[\cL(\theta_T)]
\le\EE[\cL(\theta_0)]
-\sum_{t=0}^{T-1}
\frac{\EE\|\nabla\cL(\theta_t)\|_2^4}
{2L\left(\EE\|\nabla\cL(\theta_t)\|_2^2
+N_t^{-1}\tr(\bH(\theta_t))\right)}.
\end{equation}
\end{thm}
\begin{proof}
Substituting the optimal learning rate into
Equation~\eqref{eq:upper_bound}, the contribution of iteration $t$ is
\begin{align}
&-\frac{1}{L}
\frac{\EE\|\nabla\cL(\theta_t)\|_2^4}
{\EE\|\nabla\cL(\theta_t)\|_2^2
+N_t^{-1}\tr(\bH(\theta_t))}+\frac{1}{2L}
\frac{\EE\|\nabla\cL(\theta_t)\|_2^4}
{\EE\|\nabla\cL(\theta_t)\|_2^2
+N_t^{-1}\tr(\bH(\theta_t))}=-\frac{\EE\|\nabla\cL(\theta_t)\|_2^4}
{2L\left(\EE\|\nabla\cL(\theta_t)\|_2^2
+N_t^{-1}\tr(\bH(\theta_t))\right)}.
\end{align}
Summing over $t=0,\ldots,T-1$ proves the stated bound.
\end{proof}

\begin{lem}
\label{lem:a.8}
Under Assumptions~\ref{ass:3} and~\ref{ass:4},
\begin{equation}
\tr(\bH(\theta))\le4B^2M.
\end{equation}
\end{lem}
\begin{proof}
By the definition of $\bH(\theta)$,
\begin{align}
\tr(\bH(\theta))
={}&\EE\!\Bigg[\Bigg\|
\nabla_\theta\log\pi_\theta(o|q)
\left(F(q,o)-b_\theta(q)
-\beta\log\frac{\pi_\theta(o|q)}{\pi_{\mathrm{ref}}(o|q)}\right)-\nabla_\theta J(\theta)\Bigg\|_2^2\Bigg]\nonumber\\
\le{}&\EE\!\left[\|\nabla_\theta\log\pi_\theta(o|q)\|_2^2
\left(F(q,o)-b_\theta(q)
-\beta\log\frac{\pi_\theta(o|q)}{\pi_{\mathrm{ref}}(o|q)}\right)^2\right].
\label{eq:app_H_second_moment_ref}
\end{align}
The inequality follows from
$\Var(X)=\EE\|X\|_2^2-\|\EE X\|_2^2\le\EE\|X\|_2^2$.
By Assumption~\ref{ass:3},
\begin{equation}
\left|F(q,o)-b_\theta(q)-\beta\log\frac{\pi_\theta(o|q)}{\pi_{\mathrm{ref}}(o|q)}\right|\le\left|F(q,o)-\beta\log\frac{\pi_\theta(o|q)}{\pi_{\mathrm{ref}}(o|q)}\right|+|b_\theta(q)|\le2B.
\end{equation}
Combining this inequality with
Equation~\eqref{eq:app_H_second_moment_ref} and Assumption~\ref{ass:4}
gives
\begin{equation}
\tr(\bH(\theta))\le4B^2\EE\|\nabla_\theta\log\pi_\theta(o|q)\|_2^2\le4B^2M.
\end{equation}
\end{proof}

\begin{thm}[Convergence Analysis]
\label{thm:a.7}
Under the optimal learning-rate schedule,
\begin{equation}
\frac{1}{T}\sum_{t=0}^{T-1}\EE\|\nabla\cL(\theta_t)\|_2^2
=O(T^{-1/2}).
\end{equation}
\end{thm}
\begin{proof}
By the preceding theorem and the nonnegativity of $\cL$,
\begin{align}
&\sum_{t=0}^{T-1}
\frac{\EE\|\nabla\cL(\theta_t)\|_2^4}
{\EE\|\nabla\cL(\theta_t)\|_2^2
+N_t^{-1}\tr(\bH(\theta_t))}\le2L\left(\EE[\cL(\theta_0)]-\EE[\cL(\theta_T)]\right)
\le2L\EE[\cL(\theta_0)].
\label{eq:app_convergence_first_ref}
\end{align}
The right-hand side is independent of $T$. Applying Cauchy--Schwarz to
the two sequences
\begin{equation}
\frac{\EE\|\nabla\cL(\theta_t)\|_2^2}
{\sqrt{\EE\|\nabla\cL(\theta_t)\|_2^2
+N_t^{-1}\tr(\bH(\theta_t))}}
\quad\text{and}\quad
\sqrt{\EE\|\nabla\cL(\theta_t)\|_2^2
+N_t^{-1}\tr(\bH(\theta_t))}
\end{equation}
gives
\begin{align}
&\left(\sum_{t=0}^{T-1}
\EE\|\nabla\cL(\theta_t)\|_2^2\right)^2\le
\left(\sum_{t=0}^{T-1}
\frac{\EE\|\nabla\cL(\theta_t)\|_2^4}
{\EE\|\nabla\cL(\theta_t)\|_2^2
+N_t^{-1}\tr(\bH(\theta_t))}\right)
\left(\sum_{t=0}^{T-1}
\left(\EE\|\nabla\cL(\theta_t)\|_2^2
+\frac{1}{N_t}\tr(\bH(\theta_t))\right)\right).
\label{eq:app_convergence_cs_ref}
\end{align}
Combining Equations~\eqref{eq:app_convergence_first_ref} and
\eqref{eq:app_convergence_cs_ref}, and absorbing constants independent
of $T$ into the big-$O$ notation, gives
\begin{equation}
\frac{\left(\sum_{t=0}^{T-1}
\EE\|\nabla\cL(\theta_t)\|_2^2\right)^2}
{\sum_{t=0}^{T-1}\left(
\EE\|\nabla\cL(\theta_t)\|_2^2
+N_t^{-1}\tr(\bH(\theta_t))\right)}=O(1).
\label{eq:app_convergence_ratio_ref}
\end{equation}
Since $N_t\ge1$, Lemma~\ref{lem:a.8} implies
\begin{equation}
\left(\sum_{t=0}^{T-1}\EE\|\nabla\cL(\theta_t)\|_2^2\right)^2\le O(1)\left(\sum_{t=0}^{T-1}\EE\|\nabla\cL(\theta_t)\|_2^2+4TB^2M\right).
\label{eq:app_convergence_quadratic_ref}
\end{equation}
Solving this quadratic inequality gives
\begin{equation}
\sum_{t=0}^{T-1}\EE\|\nabla\cL(\theta_t)\|_2^2
=O(\sqrt{T}).
\end{equation}
Dividing both sides by $T$ proves
\begin{equation}
\frac{1}{T}\sum_{t=0}^{T-1}\EE\|\nabla\cL(\theta_t)\|_2^2
=O(T^{-1/2}).
\end{equation}
\end{proof}

\subsection{Proofs for Section~\ref{sec:4.4}}
\label{app:A.4}

\begin{thm}[KL-Regularized Optimal Baseline]
\label{thm:a.9}
The variance-optimal baseline is
\begin{align}
b_\theta(q)
={}&\frac{\EE_{o\sim\pi_\theta(\cdot|q)}\!\left[
\|\nabla_\theta\log\pi_\theta(o|q)\|_2^2
\left(F(q,o)-\beta\log\frac{\pi_\theta(o|q)}
{\pi_{\mathrm{ref}}(o|q)}\right)\right]}{\EE_{o\sim\pi_\theta(\cdot|q)}
[\|\nabla_\theta\log\pi_\theta(o|q)\|_2^2]}.
\end{align}
\end{thm}
\begin{proof}
The mean gradient is independent of $b_\theta(q)$ because
\begin{equation}
\EE_{o\sim\pi_\theta(\cdot|q)}\!\left[
 b_\theta(q)\nabla_\theta\log\pi_\theta(o|q)\right]=0.
\end{equation}
Therefore, minimizing $\tr(\bH(\theta))$ is equivalent to minimizing
\begin{align}
\EE_{q\sim D}\EE_{o\sim\pi_\theta(\cdot|q)}\!\Bigg[
&\|\nabla_\theta\log\pi_\theta(o|q)\|_2^2\left(F(q,o)-\beta\log\frac{\pi_\theta(o|q)}
{\pi_{\mathrm{ref}}(o|q)}-b_\theta(q)\right)^2\Bigg].
\label{eq:app_baseline_objective_ref}
\end{align}
The baseline can be chosen separately for each query, so we minimize the
inner conditional expectation for a fixed $q$. Expanding the square, the
terms that depend on $b_\theta(q)$ are
\begin{align}
&b_\theta(q)^2
\EE_{o\sim\pi_\theta(\cdot|q)}
[\|\nabla_\theta\log\pi_\theta(o|q)\|_2^2]-2b_\theta(q)
\EE_{o\sim\pi_\theta(\cdot|q)}\!\left[
\|\nabla_\theta\log\pi_\theta(o|q)\|_2^2
\left(F(q,o)-\beta\log\frac{\pi_\theta(o|q)}
{\pi_{\mathrm{ref}}(o|q)}\right)\right].
\label{eq:app_baseline_quadratic_ref}
\end{align}
Differentiating Equation~\eqref{eq:app_baseline_quadratic_ref} with
respect to $b_\theta(q)$ gives
\begin{align}
&2b_\theta(q)
\EE_{o\sim\pi_\theta(\cdot|q)}
[\|\nabla_\theta\log\pi_\theta(o|q)\|_2^2]\nonumber-2\EE_{o\sim\pi_\theta(\cdot|q)}\!\left[
\|\nabla_\theta\log\pi_\theta(o|q)\|_2^2
\left(F(q,o)-\beta\log\frac{\pi_\theta(o|q)}
{\pi_{\mathrm{ref}}(o|q)}\right)\right].
\end{align}
Setting this derivative to zero and solving for $b_\theta(q)$ yields
\begin{align}
b_\theta(q)
={}&\frac{\EE_{o\sim\pi_\theta(\cdot|q)}\!\left[
\|\nabla_\theta\log\pi_\theta(o|q)\|_2^2
\left(F(q,o)-\beta\log\frac{\pi_\theta(o|q)}
{\pi_{\mathrm{ref}}(o|q)}\right)\right]}{\EE_{o\sim\pi_\theta(\cdot|q)}
[\|\nabla_\theta\log\pi_\theta(o|q)\|_2^2]}.
\end{align}
The coefficient of $b_\theta(q)^2$ in
Equation~\eqref{eq:app_baseline_quadratic_ref} is nonnegative, so this
stationary point is a global minimizer. Setting $\beta=0$ recovers the
unregularized optimal baseline in the reference analysis.
\end{proof}

\section{Pseudocode for the Methodological Components}
\label{app:method_algorithms}
This section provides the pseudocode for the two modular components and their combined OBLR-PO update introduced in Section~\ref{sec:5}.

\setcounter{algorithm}{0}
\begin{algorithm}[htb]
\caption{Policy Optimization with the Variance-Optimal Baseline}
\label{alg:optimal_baseline}
\begin{algorithmic}[1]
\REQUIRE Group rollouts $\{(q,\{o_i,r_i\}_{i=1}^{G_t})\}$, reference policy $\pi_{\mathrm{ref}}$, KL coefficient $\beta$, and fixed learning rate $\eta$
\FOR{$i=1,\dots,G_t$}
    \STATE Compute $R_i$ using Equation~\eqref{eq:empirical_kl_reward}
    \STATE Compute $s_i=\|\nabla_\theta\log\pi_\theta(o_i|q)\|_2^2$
\ENDFOR
\FOR{$i=1,\dots,G_t$}
    \STATE Compute $\hat b_\theta(q,o_i)$ using Equation~\eqref{eq:empirical_beta_baseline}
    \STATE Set $\hat A(q,o_i)=R_i-\hat b_\theta(q,o_i)$
\ENDFOR
\STATE Construct $\hat g_t$ from the resulting advantages
\STATE Update $\theta\leftarrow\theta+\eta\hat g_t$
\end{algorithmic}
\end{algorithm}

\begin{algorithm}[htb]
\caption{Policy Optimization with an SNR-Adaptive Learning Rate}
\label{alg:adaptive_lr}
\begin{algorithmic}[1]
\REQUIRE Group rollouts $\{(q,\{o_i,r_i\}_{i=1}^{G_t})\}$, baseline rule $b$, base rate $\eta_0$, and learning-rate band $[\eta_{\min},\eta_{\max}]$
\FOR{$i=1,\dots,G_t$}
    \STATE Compute $\hat A(q,o_i)$ using $b$
\ENDFOR
\STATE Construct total-gradient samples from the advantages
\STATE Estimate $\widehat{\operatorname{SNR}}(\theta_t)$ from these samples
\STATE Compute $\hat\eta_t$ using Equation~\eqref{eq:empirical_adaptive_lr}
\STATE Aggregate the samples into $\hat g_t$
\STATE Update $\theta\leftarrow\theta+\hat\eta_t\hat g_t$
\end{algorithmic}
\end{algorithm}


\clearpage 

\section{Experimental Setups}
\label{app:setup}

We provide the full experimental configuration below. All runs use Qwen3-4B-Base trained on DeepMath-103K with verifiable rule-based rewards, optimized with SGD under FSDP and a vLLM rollout backend. Table~\ref{tab:hparams} lists the shared hyperparameters; the SNR-adaptive learning-rate band and per-method configuration notes follow.

\begin{table}[h]
\centering
\caption{Hyperparameters for Qwen3-4B-Base experiments.}
\label{tab:hparams}
\begin{tabular}{ll}
\toprule
Setting & Value \\
\midrule
Base model & Qwen3-4B-Base \\
Training set & DeepMath-103K \\
Evaluation sets & MATH-500, GSM8K, \\
                & OlympiadBench, AMC23 \\
Optimizer & SGD (momentum $0$, weight decay $0$) \\
Group size $G$ & $8$ \\
Prompt batch size & $32$ \\
Mini-batch size & $32$ \\
Max prompt / response length & $1024$ / $1024$ \\
KL regularizer type & \texttt{low\_var\_kl} \\
KL coefficient $\beta$ & $10^{-3}$ \\
Entropy coefficient & $0$ \\
Gradient-norm clip & $1.0$ \\
Fixed learning rate & $10^{-2}$ \\
SNR-adaptive LR band & $[7\times10^{-3},\, 2\times10^{-2}]$ \\
Rollout backend / TP & vLLM / $2$ \\
Hardware & $4$ A800 GPUs, FSDP \\
Evaluation metric & \texttt{Accuracy (Pass@1)}\\
\bottomrule
\end{tabular}
\end{table}

\paragraph{SNR-adaptive learning rate.}
The effective learning rate is $\eta_t=\mathrm{clamp}\!\left(\eta_0\cdot\frac{M\,\widehat{\mathrm{SNR}}}{1+M\,\widehat{\mathrm{SNR}}},\,\eta_{\min},\,\eta_{\max}\right)$, where $M$ equals the mini-batch size and $\widehat{\mathrm{SNR}}$ is estimated from per-micro-batch gradient deltas during gradient accumulation as shown in Section~\ref{app:estimators}.

\section{Computation of the Gradient and SNR Estimators}
\label{app:estimators}
This section makes precise how the two central estimators of OBLR-PO---the gradient-weighted advantage and the SNR-adaptive learning rate---are computed. We follow the notation of the main text: for a query $q$, the old policy $\pi_{\theta_{\mathrm{old}}}$ samples a group of $G$ outputs $\{o_j\}_{j=1}^{G}$ with scalar rewards $r_j=F(q,o_j)$; $\pi_\theta$ and $\pi_{\mathrm{ref}}$ are the current and reference policies, and $\beta\ge0$ is the KL coefficient. We abbreviate the per-sequence \emph{score norm} by
\begin{equation}
    s_j \;=\; \bigl\|\nabla_\theta\log\pi_\theta(o_j|q)\bigr\|_2^2 .
    \label{eq:app_score_norm}
\end{equation}

\subsection{Gradient Estimator}
\label{app:estimators.grad}

\paragraph{General Formulation.}
The gradient estimator weights each sampled trajectory's score function by its advantage, where the advantage itself combines a KL-regularized reward with the variance-optimal baseline. For each output $o_j$, define the KL-regularized reward
\begin{equation}
    R_j \;=\; F(q,o_j)-\beta\log\frac{\pi_\theta(o_j|q)}{\pi_{\mathrm{ref}}(o_j|q)} .
    \label{eq:app_klreward}
\end{equation}
By Theorem~\ref{thm:optimal_beta_baseline}, the variance-optimal baseline for $o_i$ is the score-norm-weighted leave-one-out average of Equation~\eqref{eq:empirical_beta_baseline},
\begin{equation}
    \hat b_\theta(q,o_i)\;=\;\frac{\sum_{j\neq i}s_j\,R_j}{\sum_{j\neq i}s_j},
    \label{eq:app_weighted_loo}
\end{equation}
which reduces to the unweighted RLOO baseline $\tfrac{1}{G-1}\sum_{j\neq i}R_j$ when the score norms $s_j$ are equal, and the corresponding advantage is $\hat A(q,o_i)=R_i-\hat b_\theta(q,o_i)$. Given a batch of $N_t$ queries, each with $G_t$ sampled outputs, the policy gradient is estimated by the importance-weighted score-function form of Equation~\eqref{eq:beta_gradient_estimator},
\begin{equation}
    \hat{g_t}=\frac{1}{N_t G_t}\sum_{j=1}^{N_t}\sum_{i=1}^{G_t}\rho_{i,j}\,\nabla_\theta\log\pi_\theta(o_{i,j}|q_j)\,\hat A(q_j,o_{i,j}),
    \label{eq:app_grad_est}
\end{equation}
where $\rho_{i,j}=\pi_\theta(o_{i,j}|q_j)/\pi_{\theta_{\mathrm{old}}}(o_{i,j}|q_j)$ is the importance-sampling ratio correcting for sampling under the old policy $\pi_{\theta_{\mathrm{old}}}$, clipped to a half-width $\epsilon$ about $1$. Each sampled trajectory thus contributes its score function $\nabla_\theta\log\pi_\theta(o_{i,j}|q_j)$ scaled by the ratio $\rho_{i,j}$ and its advantage $\hat A(q_j,o_{i,j})$, with the KL regularization carried inside $\hat A$ through the KL-regularized reward $R_{i,j}$.

\paragraph{Implementation in the Experiment}
Our experiment evaluates $\hat A(q,o_i)$ as follows.
\begin{enumerate}
    \item \emph{KL term.} The log-ratio in Equation~\eqref{eq:app_klreward} is estimated per response token by the non-negative, low-variance (k3) estimator
    \begin{equation}
        \widehat{\mathrm{KL}}=e^{\Delta}-\Delta-1,\qquad \Delta=\log\pi_{\mathrm{ref}}-\log\pi_\theta ,
        \label{eq:app_k3}
    \end{equation}
    with the KL coefficient $\beta$ of Table~\ref{tab:hparams}. The penalty is folded into the reward $R_j$ \emph{before} the baseline, so its KL contribution is itself variance-reduced by the score-norm weighting; no separate loss-level KL term is added, so the update coincides with the estimator of Equation~\eqref{eq:app_grad_est}.
    \item \emph{Score norms.} Each $s_j$ in Equation~\eqref{eq:app_score_norm} is computed over the full model by one additional backward pass per sequence on the surrogate loss $-\sum_{t}\log\pi_\theta(o_{j,t}\mid q,o_{j,1:t-1})$, whose gradient is exactly $-\nabla_\theta\log\pi_\theta(o_j|q)$; the pass retains unsharded (non-reduce-scattered) gradients so that $s_j$ is the true squared norm.
    \item \emph{Baseline.} Equation~\eqref{eq:app_weighted_loo} is evaluated directly, with a fallback to the unweighted RLOO baseline whenever the denominator $\sum_{j\neq i}s_j$ underflows.
    \item \emph{Surrogate.} The ratio $\rho_{i,j}$ is applied at the token level, its log clamped for numerical stability, and clipping is effectively disabled by taking $\epsilon$ large, so the objective keeps the raw importance ratio; the per-token losses are aggregated by a token-mean.
\end{enumerate}

\subsection{SNR Estimator}
\label{app:estimators.snr}

\paragraph{General Formulation.}
Theorem~\ref{thm:opt_lr} sets the step size from the gradient signal-to-noise ratio,
\begin{equation}
    \operatorname{SNR}(\theta)=\frac{\EE\|\nabla_\theta\cL(\theta)\|_2^2}{\tr(\bH(\theta))},
    \qquad
    \eta_t=\frac{1}{L}\cdot\frac{N_t\operatorname{SNR}(\theta_t)}{1+N_t\operatorname{SNR}(\theta_t)} .
\end{equation}
Since $L$ and $\nabla_\theta\cL$ are unavailable at run time, we estimate the signal and noise directly from the gradients produced during gradient accumulation. Suppose the mini-batch gradient is accumulated over $K$ micro-batches; let $\hat g$ be the running accumulated mean gradient, $\delta_k$ the increment contributed by micro-batch $k$, and $n_k,\tau_k$ its number of sequences and (masked) tokens, with mini-batch totals $B=\sum_k n_k$ and $N_\tau=\sum_k\tau_k$. Treating the $K$ increments as approximately i.i.d.\ samples of the mini-batch gradient---with a token rescaling $(N_\tau/\tau_k)^2$ that undoes the token-mean normalization---we form
\begin{align}
    S_1 &= \sum_{k=1}^{K} n_k\Bigl(\tfrac{N_\tau}{\tau_k}\Bigr)^{2}\|\delta_k\|_2^2,\\
    \tr(\hat\Sigma) &= \max\!\Bigl(\tfrac{S_1-B\|\hat g\|_2^2}{K-1},\,\varepsilon\Bigr),
    \label{eq:app_noise}\\
    \|\hat G\|_2^2 &= \max\!\bigl(\|\hat g\|_2^2-\tfrac{1}{B}\tr(\hat\Sigma),\,0\bigr),
    \label{eq:app_signal}\\
    \widehat{\operatorname{SNR}} &= \frac{\|\hat G\|_2^2}{\tr(\hat\Sigma)} ,
\end{align}
where Equation~\eqref{eq:app_noise} is a Bessel-corrected sample variance floored at $\varepsilon>0$, and Equation~\eqref{eq:app_signal} debiases the squared mean-gradient norm by subtracting the sampling variance of the mean. The learning rate then follows the clamped SNR schedule of Appendix~\ref{app:setup},
\begin{align}
    \mathrm{coeff} &= \mathrm{clamp}\!\Bigl(\tfrac{M\,\widehat{\operatorname{SNR}}}{1+M\,\widehat{\operatorname{SNR}}},\,\mathrm{coeff}_{\min},\,1\Bigr),\\
    \eta_t &= \mathrm{clamp}\bigl(\eta_0\cdot\mathrm{coeff},\,\eta_{\min},\,\eta_{\max}\bigr),
\end{align}
with prefactor $M$, coefficient floor $\mathrm{coeff}_{\min}$, base rate $\eta_0$, and effective-rate band $[\eta_{\min},\eta_{\max}]$.

\paragraph{Implementation in the Experiment.}
In our experiment, the increments $\delta_k$ are read off the data-parallel-synchronized running gradient: after each micro-batch backward we record $\|\delta_k\|_2^2$ from a transient \texttt{fp32} copy of the sharded gradient, together with $n_k$ and $\tau_k$, and sum these across data-parallel ranks so that $\hat g$, $\|\hat G\|_2^2$, and $\tr(\hat\Sigma)$ are global mini-batch quantities. The prefactor $M$ is set to the mini-batch size in prompts, the coefficient floor to $\mathrm{coeff}_{\min}=0$, and the noise floor to $\varepsilon=10^{-12}$; the base rate $\eta_0$ and the band $[\eta_{\min},\eta_{\max}]$ take the values in Table~\ref{tab:hparams}. Gradient-norm clipping still applies to $\hat g$ before the scaled step, and the mechanism is disabled under sequence parallelism or a mixed-precision gradient scaler.


\clearpage
\section{Additional Results}
\label{app:experiments}
In this section, we provide additional results: the SNR-adaptive learning rate on the benchmarks not shown in the main text (Section~\ref{sec:6.2}), the OBLR-PO learning-rate ablation (Section~\ref{sec:6.4}), the average learning rate realized by the SNR-adaptive schedule, and the training-time overhead.

\subsection{SNR-Adaptive Learning Rate on MATH-500}
The main text shows the GSM8K accuracy and the training-reward curves for the SNR-adaptive learning rate (Figure~\ref{fig:lr}). Figure~\ref{fig:app-lr-math500} provides the MATH-500 accuracy curves, comparing a fixed learning rate against the SNR-Adaptive LR for RLOO, ReMax, and GRPO. Adding the SNR-Adaptive LR consistently matches or improves accuracy. RLOO and ReMax satisfy Assumption~\ref{ass:1}, whereas GRPO is reported as an additional, beyond-theory case.

\begin{figure}[ht]
\centering
\begin{subfigure}{0.32\linewidth}
    \includegraphics[width=\linewidth]{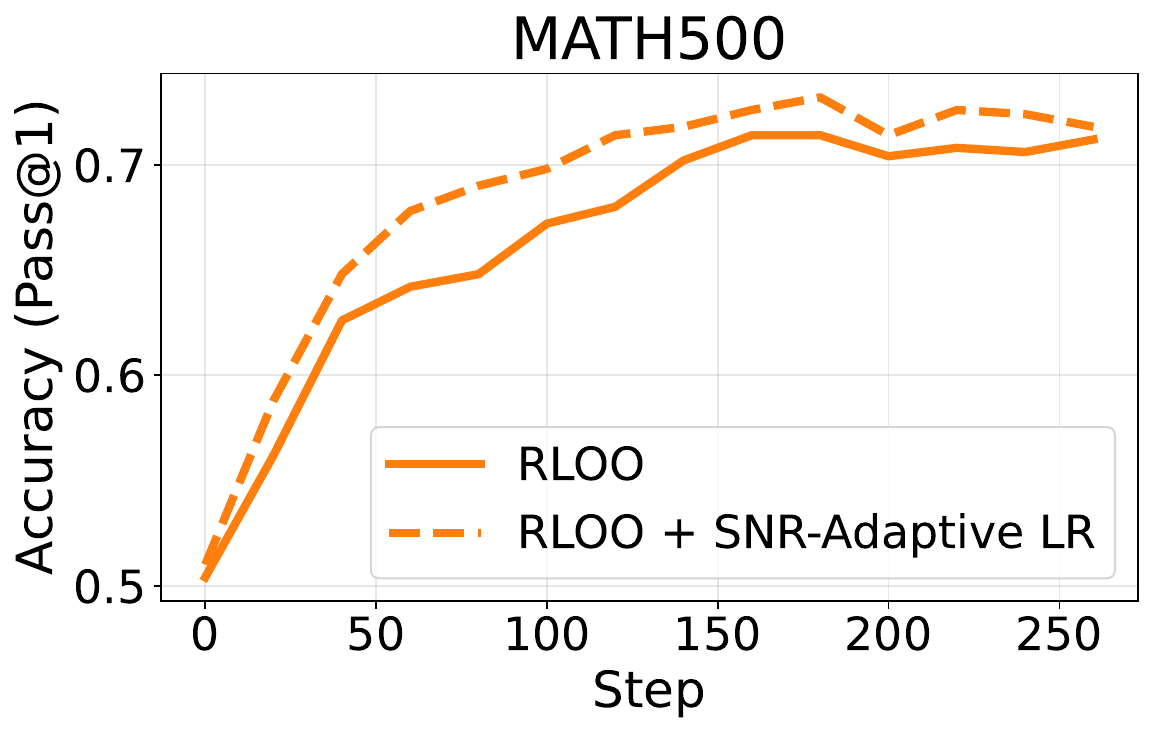}
    \caption{MATH-500, RLOO}
\end{subfigure}
\hfill
\begin{subfigure}{0.32\linewidth}
    \includegraphics[width=\linewidth]{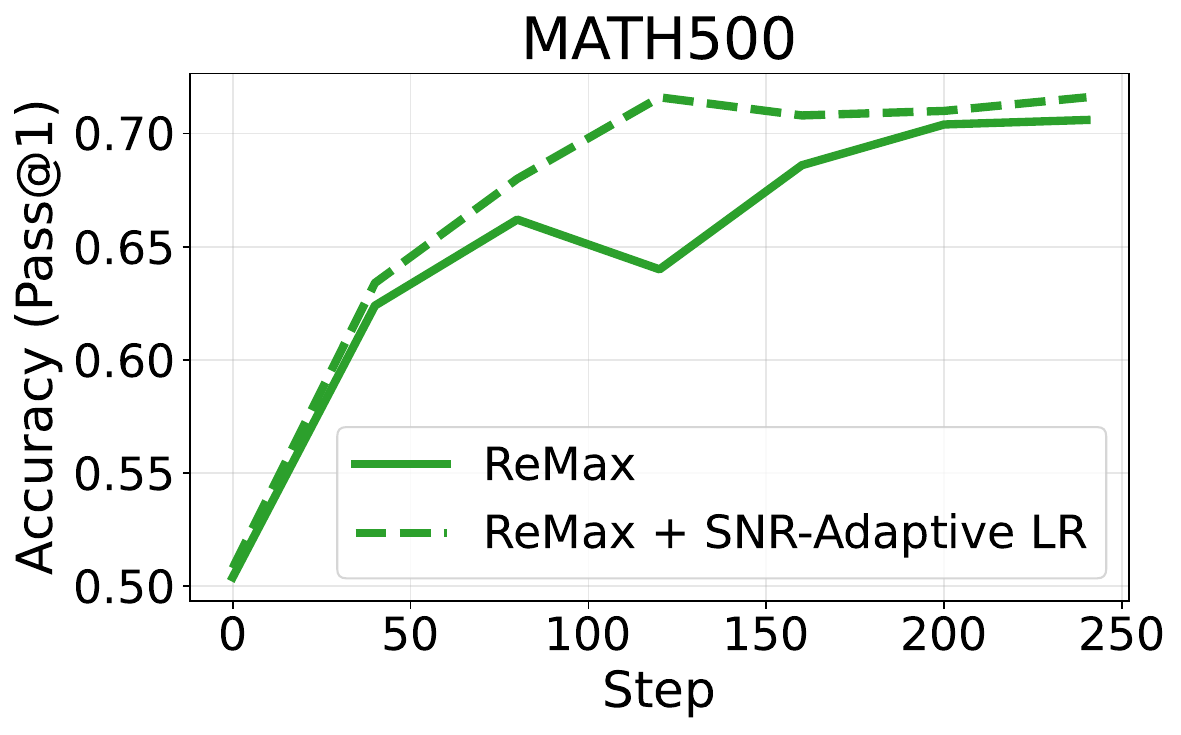}
    \caption{MATH-500, ReMax}
\end{subfigure}
\hfill
\begin{subfigure}{0.32\linewidth}
    \includegraphics[width=\linewidth]{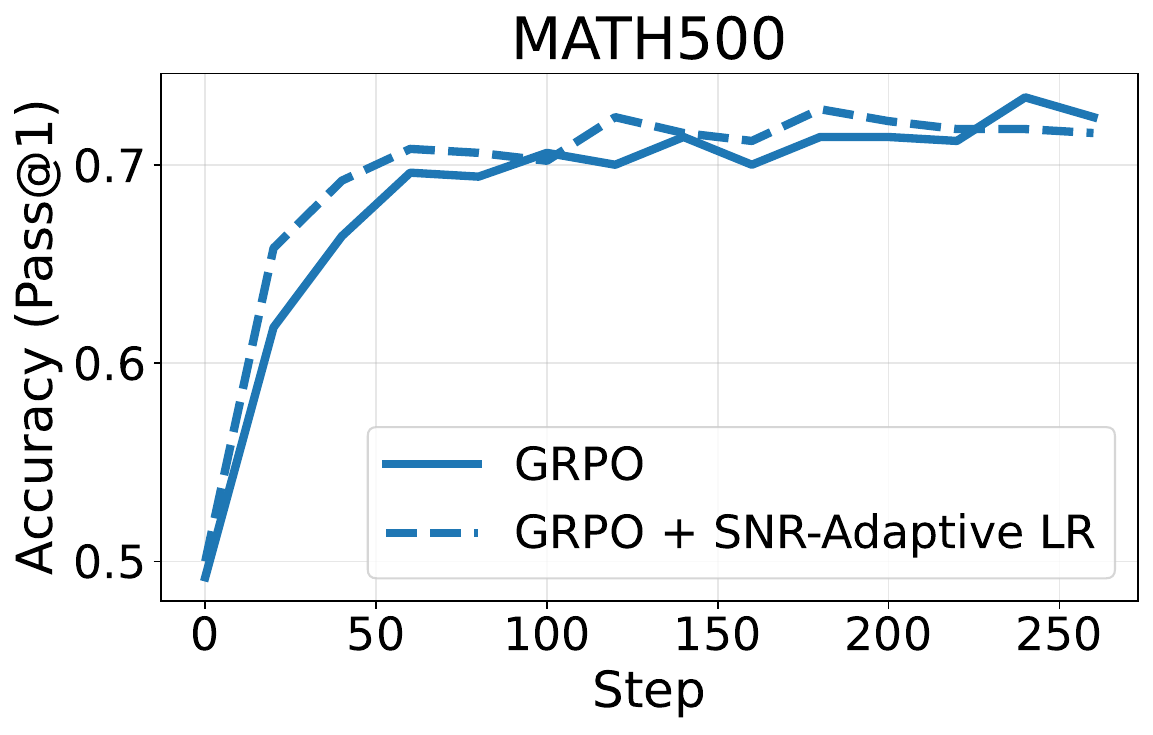}
    \caption{MATH-500, GRPO}
\end{subfigure}
\caption{Pass@1 accuracy on MATH-500 with a fixed learning rate versus the SNR-Adaptive LR, for RLOO, ReMax, and GRPO.}
\label{fig:app-lr-math500}
\end{figure}

\subsection{SNR-Adaptive Learning Rate on OlympiadBench}
Figure~\ref{fig:app-lr-olympiad} provides the OlympiadBench accuracy curves, comparing a fixed learning rate against the SNR-Adaptive LR for RLOO and GRPO. As on the other benchmarks, adding the SNR-Adaptive LR matches or improves accuracy.

\begin{figure}[ht]
\centering
\begin{subfigure}{0.49\linewidth}
    \includegraphics[width=\linewidth]{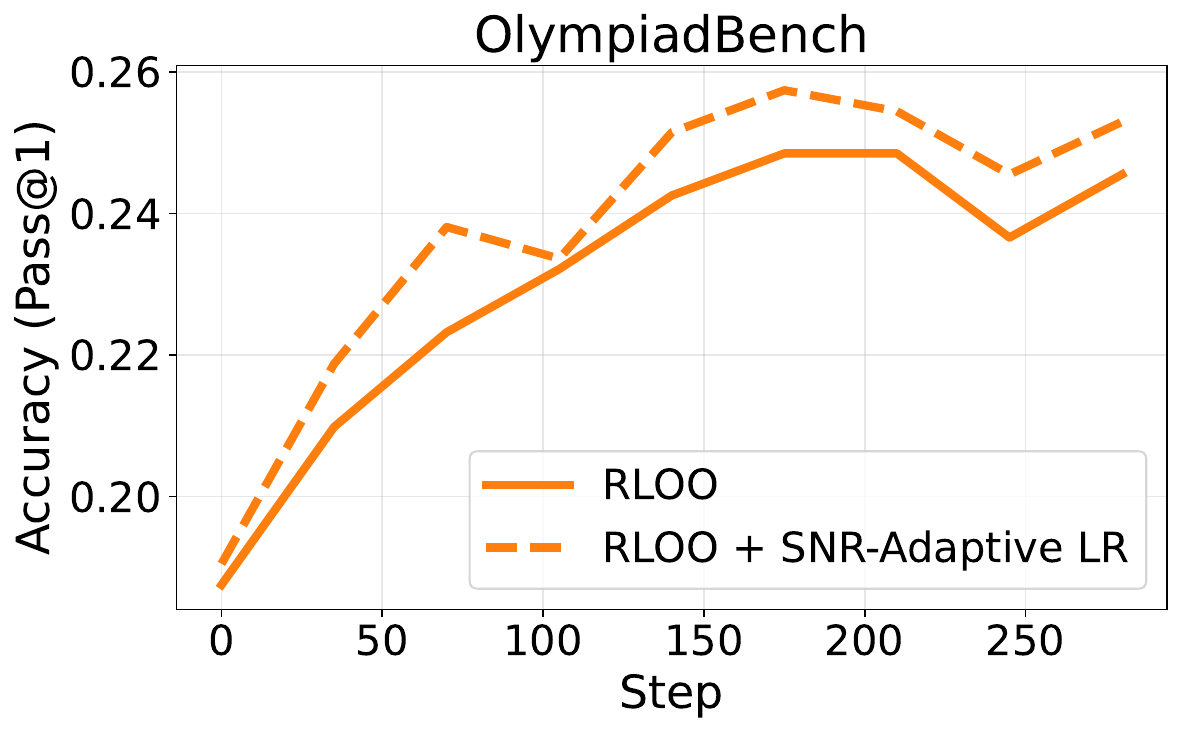}
    \caption{OlympiadBench, RLOO}
\end{subfigure}
\hfill
\begin{subfigure}{0.49\linewidth}
    \includegraphics[width=\linewidth]{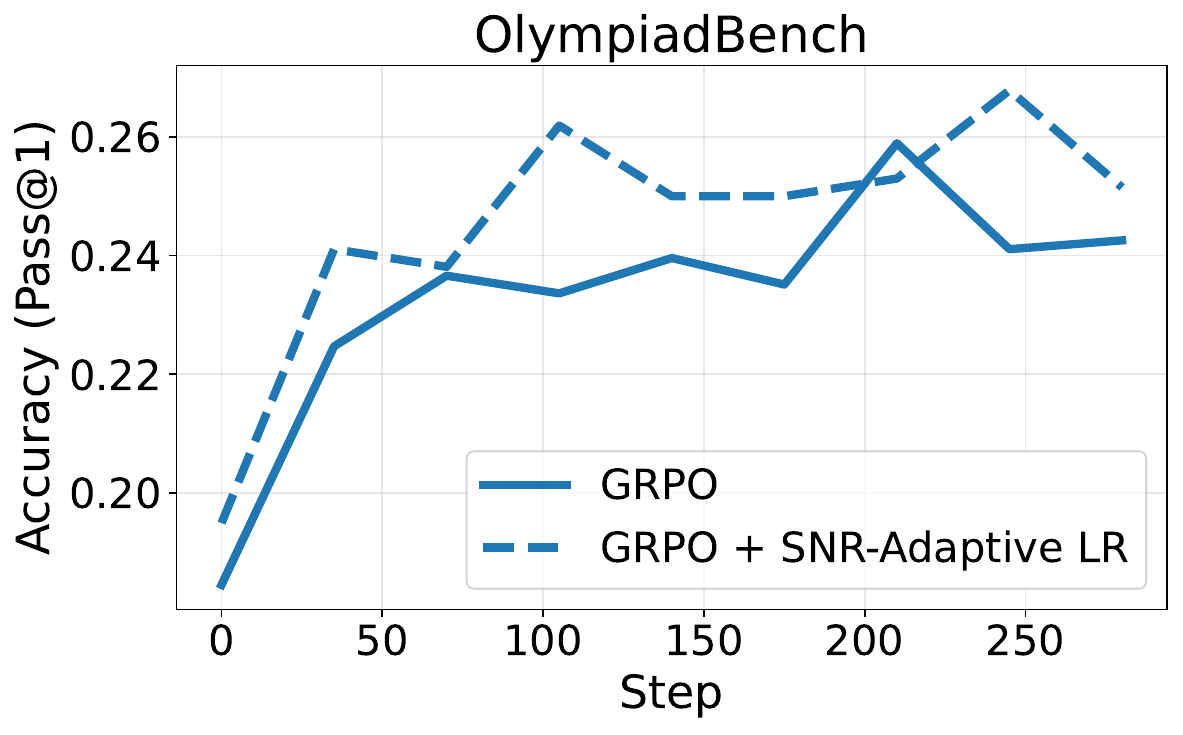}
    \caption{OlympiadBench, GRPO}
\end{subfigure}
\caption{Pass@1 accuracy on OlympiadBench with a fixed learning rate versus the SNR-Adaptive LR, for RLOO and GRPO.}
\label{fig:app-lr-olympiad}
\end{figure}

\subsection{OBLR-PO Learning-Rate Ablation}
Figure~\ref{fig:app-oblrpo-incr} isolates the incremental effect of the SNR-adaptive learning rate within OBLR-PO, comparing the variance-optimal baseline alone against the full OBLR-PO (variance-optimal baseline ${+}$ SNR-Adaptive LR). The learning-rate component yields a further gain over the baseline-only variant, as discussed in Section~\ref{sec:6.4}.

\begin{figure}[ht]
    \centering
    \begin{subfigure}{0.51\linewidth}
        \centering
        \includegraphics[width=\linewidth]{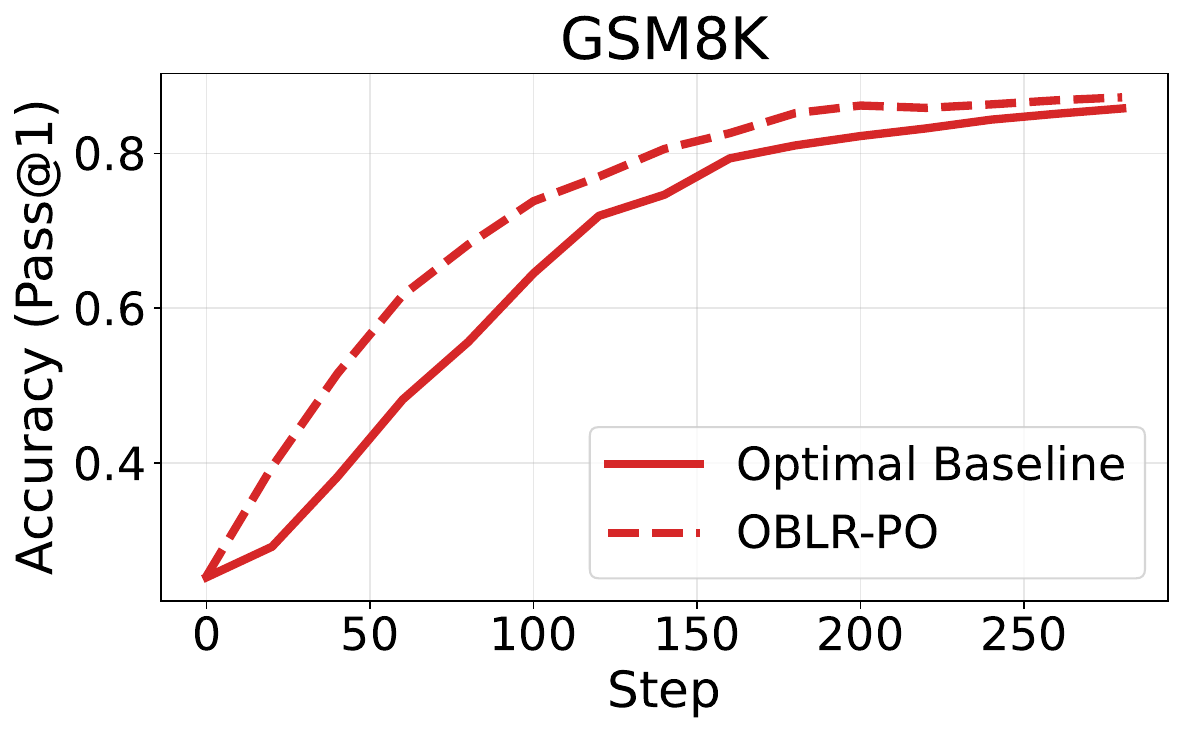}
        \caption{GSM8K accuracy}
    \end{subfigure}
    \hfill
    \begin{subfigure}{0.47\linewidth}
        \centering
        \includegraphics[width=\linewidth]{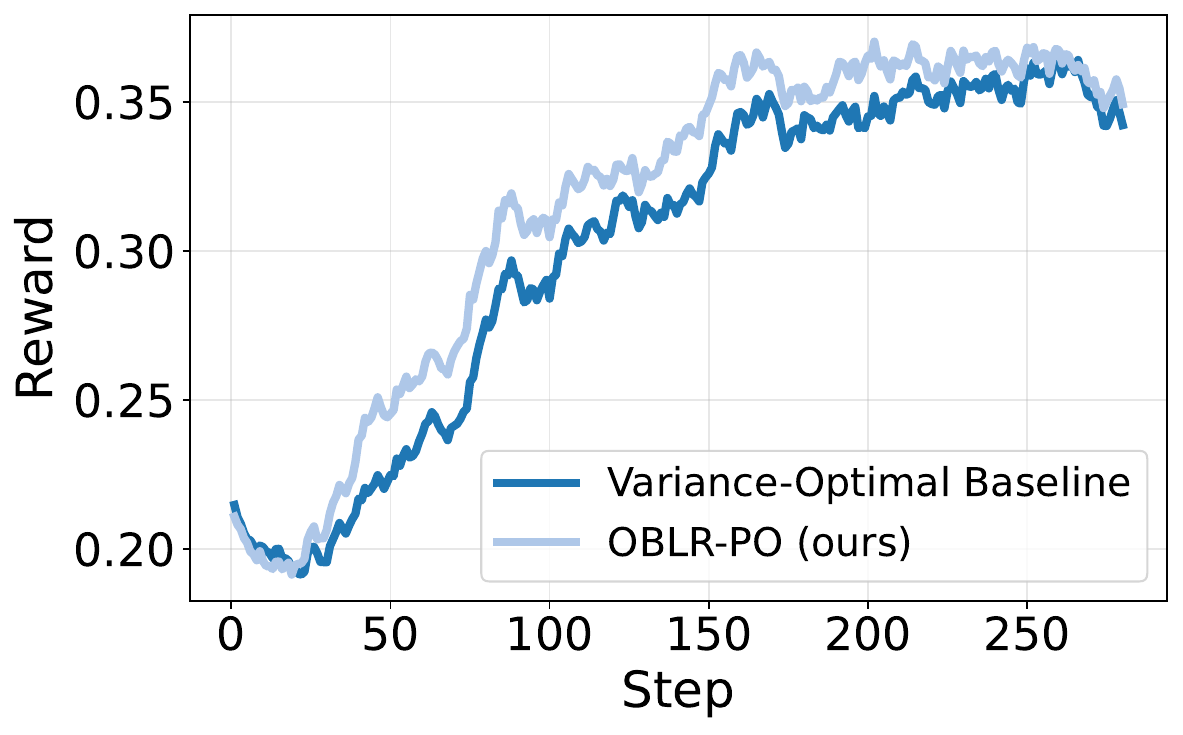}
        \caption{Training reward}
    \end{subfigure}
    \caption{Incremental effect of the SNR-adaptive learning rate within OBLR-PO: the variance-optimal baseline alone versus the full OBLR-PO (variance-optimal baseline ${+}$ SNR-Adaptive LR).}
    \label{fig:app-oblrpo-incr}
\end{figure}

\subsection{Average Learning Rate of the SNR-Adaptive Schedule}
Because the SNR-adaptive schedule (Equation~\eqref{eq:empirical_adaptive_lr}) varies the step size over training, we verify that its \emph{average} effective learning rate stays close to the fixed learning rate $1\times10^{-2}$ used by the baseline runs, so that the two are compared at a matched step-size budget. Table~\ref{tab:avglr} reports the mean SNR-adaptive learning rate over training for each SNR-LR experiment. All averages fall near $1\times10^{-2}$ and well within the clamp band $[7\times10^{-3},\,2\times10^{-2}]$, confirming that the SNR-adaptive and fixed-learning-rate runs operate at comparable effective step sizes and that their accuracy and reward comparisons are fair.

\begin{table}[ht]
\centering
\caption{Mean SNR-adaptive learning rate over training for each SNR-adaptive learning rate experiment.}
\label{tab:avglr}
\begin{tabular}{lc}
\toprule
Experiment & Mean SNR-adaptive LR \\
\midrule
GRPO ${+}$ SNR-adaptive LR  & $1.02\times10^{-2}$ \\ 
RLOO ${+}$ SNR-adaptive LR  & $ 1.10\times10^{-2}$ \\ 
ReMax ${+}$ SNR-adaptive LR & $ 1.10\times10^{-2}$ \\ 
OBLR-PO & $ 1.07\times10^{-2}$ \\ 
\bottomrule
\end{tabular}
\end{table}

\subsection{Training-Time Overhead}
The SNR-adaptive learning rate estimates the gradient signal-to-noise ratio from per-micro-batch gradient statistics gathered during gradient accumulation (Appendix~\ref{app:estimators.snr}), and the Variance-Optimal Baseline adds one extra backward pass per sequence to compute the score norms (Appendix~\ref{app:estimators.grad}); both could in principle add wall-clock cost. Table~\ref{tab:overhead} reports the mean per-step training time, measured on $4\times$A800 GPUs, and shows that both components are cheap. First, the SNR-adaptive schedule is effectively free: adding it to any base estimator changes the per-step time by only $-1.9\%$ to $+1.0\%$, within measurement noise. Second, the Variance-Optimal Baseline is also inexpensive: its extra backward pass raises the per-step time only from roughly $39$\,s (RLOO/ReMax/GRPO) to roughly $67$\,s---an under-$2\times$ constant factor that stays in the same order of magnitude and is comfortably offset by the variance reduction and accuracy gains it delivers (Section~\ref{sec:6.3}).

\begin{table}[ht]
\centering
\caption{Per-step training time on $4\times$A800 GPUs, with and without the SNR-adaptive learning rate.}
\label{tab:overhead}
\begin{tabular}{lc}
\toprule
Method & Mean seconds/step \\
\midrule
GRPO             & $38.70{\scriptstyle\,\pm0.77}$ \\
GRPO ${+}$ SNR-LR  & $37.97{\scriptstyle\,\pm0.74}$ \\
\midrule
RLOO             & $39.90{\scriptstyle\,\pm0.46}$ \\
RLOO ${+}$ SNR-LR  & $40.31{\scriptstyle\,\pm1.37}$ \\
\midrule
ReMax            & $39.07{\scriptstyle\,\pm0.39}$ \\
ReMax ${+}$ SNR-LR & $38.73{\scriptstyle\,\pm0.34}$ \\
\midrule
Variance-Optimal Baseline           & $66.66{\scriptstyle\,\pm1.31}$ \\
OBLR-PO & $66.90{\scriptstyle\,\pm2.13}$ \\
\bottomrule
\end{tabular}
\end{table}

\end{document}